\documentclass[11pt]{article}
\usepackage{fullpage}

\usepackage[utf8]{inputenc} 
\usepackage[T1]{fontenc}    
\usepackage{url}            
\usepackage{booktabs}       
\usepackage{amsfonts}       
\usepackage{microtype}      

\usepackage{hyperref}
\hypersetup{pdftex,colorlinks=true,allcolors=blue} 
\usepackage{hypcap} 

\usepackage{amsmath, amsthm, amssymb}
\usepackage{graphicx}
\usepackage{cite}
\usepackage{mathrsfs}
\usepackage{comment}
\usepackage{enumitem} 
\usepackage{lmodern}

\newtheorem{theorem}{Theorem}
\newtheorem{lemma}{Lemma}

\newtheorem{corollary}{Corollary}
\newtheorem{definition}{Definition}
\newcommand{\R}{\mathbb{R}}
\newcommand{\E}{\mathbb{E}}

\newcommand{\Var}{{\rm Var}}
\newcommand{\I}{\mathbf{1}}

\DeclareMathOperator*{\argmax}{arg\,max}
\DeclareMathOperator*{\argmin}{arg\,min}

\def\sA{{\mathsf A}}

\def\sU{{\mathsf U}}

\def\sW{{\mathsf W}}
\def\sX{{\mathsf X}}
\def\sY{{\mathsf Y}}
\def\sZ{{\mathsf Z}}

\def\rd{{\rm d}}
\def\PP{{\mathbb P}}

\def\deq{\triangleq}
\def\wh#1{{\widehat{#1}}}

\def\eps{\varepsilon}

\newcommand\blfootnote[1]{%
	\begingroup
	\renewcommand\thefootnote{}\footnote{#1}%
	\addtocounter{footnote}{-1}%
	\endgroup
}

\title{Continuity of Generalized Entropy and Statistical Learning}
\author{Aolin Xu} 
\date{}

\begin{document}
	
	\maketitle
	
	\begin{abstract}
		We study the continuity property of the generalized entropy as a function of the underlying probability distribution, defined with an action space and a loss function, and use this property to answer the basic questions in statistical learning theory: the excess risk analyses for various learning methods.
		We first derive upper and lower bounds for the entropy difference of two distributions in terms of several commonly used $f$-divergences, the Wasserstein distance, a distance that depends on the action space and the loss function, and the Bregman divergence generated by the entropy, which also induces bounds in terms of the Euclidean distance between the two distributions. 
		Examples are given along with the discussion of each general result,
		comparisons are made with the existing entropy difference bounds, and new mutual information upper bounds are derived based on the new results.
		We then apply the entropy difference bounds to the theory of statistical learning.
		It is shown that the excess risks in the two popular learning paradigms, the frequentist learning and the Bayesian learning, both can be studied with the continuity property of different forms of the generalized entropy.
		The analysis is then extended to the continuity of generalized conditional entropy. The extension provides performance bounds for Bayes decision making with mismatched distributions. It also leads to excess risk bounds for a third paradigm of learning, where the decision rule is optimally designed under the projection of the empirical distribution to a predefined family of distributions.
		We thus establish a unified method of excess risk analysis for the three major paradigms of statistical learning, through the continuity of generalized entropy.
	\end{abstract}
	
	\blfootnote{xuaolin@gmail.com}
	\tableofcontents
	
	\section{Introduction}\label{sec:intro}
	\subsection{Generalized entropy}
	The definition of Shannon entropy can be generalized via the following statistical decision-making problem \cite{gunwald2004}.
	Let $\sZ$ be a space of outcomes, $\sA$ be a space of actions, and $\ell: \sZ\times \sA \rightarrow \R$ be a loss function.
	An outcome $Z$ is drawn from a distribution $P$ on $\sZ$. The decision-making problem is to pick an action from $\sA$ that minimizes the expected loss.
	The minimum expected loss can be used as a definition of the \emph{generalized entropy} of distribution $P$ with respect to the action space $\sA$ and the loss function $\ell$, 
	\begin{align}\label{eq:H_def}
		H_\ell(P) \deq \inf_{a\in \sA} \E_P[\ell(Z, a)] ,
	\end{align}
	which may also be written as $H_\ell(Z)$ when the distribution of $Z$ is clear. 
	When there is a need to emphasize the role of the action space, we may use the notation $H_{\sA,\ell}(P)$ or $H_{\sA,\ell}(Z)$ as well.
	Examples of the generalized entropy include:
	\begin{itemize}[leftmargin=*]
		\item 
		When $\sA$ is the family of distributions $Q$ on $\sZ$ (e.g. $Q$ is a PMF if $\sZ=\mathbb N$, or a PDF if $\sZ=\R^p$), the optimal action for the logarithmic loss $\ell(z,Q) = -\log Q(z)$ is $P$, and $H_{\log}(Z)$ is the Shannon entropy $H(Z)$ when $\sZ$ is discrete, or the differential entropy $h(Z)$ when $\sZ$ is continuous.
		\item 
		When $\sZ=\sA=\R^p$, the optimal action for the quadratic loss $\ell(z,a) = \sum_{j=1}^p(z_j-a_j)^2$ is $\E[Z]$, and $H_2(Z) = \sum_{j=1}^p \Var[Z_j]$. In particular, when $p=1$, $H_2(Z) = \Var[Z]$.
		\item 
		When $\sZ=\sA$ are discrete, the optimal action for the zero-one loss $\ell(z,a) = \I\{z\neq a\}$ is $\argmax_{z} P(z)$, and $H_{01}(Z) = 1-\max_{z\in\sZ} P(z)$.
	\end{itemize}
	
	The above decision-making problem can also be used to formulate the frequentist statistical learning problem, by letting $\sZ$ be a sample space, $\sA$ be a hypothesis space, and $P$ be an unknown distribution on $\sZ$. For any hypothesis $a\in\sA$, $\E_P[\ell(Z,a)]$ is its population risk, and $H_{\sA,\ell}(P)$ is the minimum population risk among all hypotheses in $\sA$, which would be achieved if $P$ were known. In practice, what is available is a training dataset consisting of $n$ samples drawn i.i.d.\ from $P$, with empirical distribution $\wh P_n$. The empirical risk minimization (ERM) algorithm outputs a hypothesis $a_{\wh P_n}$ that minimizes the empirical risk $\E_{\wh P_n}[\ell(Z,a)]$ among $a\in\sA$, and $H_{\sA,\ell}(\wh P_n)$ is the minimum empirical risk.
	It is one of the main goals of statistical learning theory to bound the gap between $\E_{P}[\ell(Z,a_{\wh P_n})]$ and $H_{\sA,\ell}(P)$, known as the excess risk of the ERM algorithm.
	
	\smallskip
	The generalized entropy defined in \eqref{eq:H_def} can be extended to the \emph{generalized conditional entropy}, defined via a Bayes decision-making problem based on an observation $X\in\sX$ that statistically depends on $Z$ \cite{FarniaTse16}, as
	\begin{align}\label{eq:cond_H_def}
		H_\ell(P_{Z|X}| P_X) \deq \inf_{\psi:\sX\rightarrow \sA} \E_P[\ell(Z, \psi(X))] ,
	\end{align}
	where the expectation is taken with respect to the joint distribution $P_X P_{Z|X}$ of $(X,Z)$, and the decision rule $\psi$ ranges over all mappings from $\sX$ to $\sA$ such that the expected loss is well-defined.
	The generalized conditional entropy in \eqref{eq:cond_H_def} may also be written as $H_\ell(Z|X)$ when the joint distribution is clear.
	It is also expressible in terms of the unconditional entropy,
	\begin{align}\label{eq:cond_uncond}
		H_\ell(P_{Z|X}| P_X) = \int_\sX H_\ell(P_{Z|X=x}) P_X(\rd x) .
	\end{align}
	In Bayesian inference, the generalized conditional entropy is essentially the \emph{Bayes risk}, which quantifies the minimum achievable expected loss of the inference problem, and the optimal decision rule $\psi_{\rm B}$ is known as the \emph{Bayes decision rule}. 
	Examples, in parallel to the above instantiations of the generalized unconditional entropy, include:
	\begin{itemize}[leftmargin=*]
		\item 
		For the log loss, $H_{\log}(Z|X)$ is the conditional Shannon/differential entropy, and $\psi_{\rm B}(x)$ is the posterior distribution $P_{Z|X=x}$;
		\item 
		For the quadratic loss with $\sZ=\sA=\R^p$, $H_2(Z|X) = \sum_{j=1}^p \E[\Var[Z_j|X]]$ is the minimum mean square error (MMSE) of estimating $Z$ from $X$, and $\psi_{\rm B}(x) = \E[Z|X=x]$;
		\item 
		For the zero-one loss, $H_{01}(Z|X) = 1-\int_{\sX}\max_{z\in\sZ} P_{X,Z}({\rm d}x,z)$, and $\psi_{\rm B}(x)=\argmax_{z} P_{Z|X=x}(z)$ is the maximum a-posteriori (MAP) rule.
	\end{itemize}
	
	\noindent
	
	From the above definitions and examples, we see that the performance limits of a variety of statistical inference, learning, and decision-making problems are different instantiations of the generalized entropy or the generalized conditional entropy.
	A good understanding of the properties of the generalized entropy and its conditional version can thus help us better-understand the performance limits of such problems.

	\subsection{Continuity in distribution}
	In the first part of this paper, we study the continuity property of the generalized entropy defined in \eqref{eq:H_def} in distribution $P$.
	Given $\sA$ and $\ell$, the generalized entropy $H_{\sA,\ell}(P)$ as a function of $P$ is \emph{continuous} at $P=Q$ with respect to a statistical distance $D(\cdot,\cdot)$\footnote{Throughout the paper, $D(\cdot,\cdot)$ denotes a generic statistical distance, which may not be symmetric or satisfy triangle inequality; the KL divergence will be denoted by $D(\cdot \| \cdot)$.}, if for any $\eps > 0$, there exists a $\delta>0$ such that
	\begin{align}\label{eq:def_cont}
		|H_{\sA,\ell}(P)-H_{\sA,\ell}(Q)| < \eps 
	\end{align}
	for all $P$ satisfying $D(P, Q) < \delta$.
	In plain words, $H_{\sA,\ell}(P)$ is continuous at $Q$ if $|H_{\sA,\ell}(P)-H_{\sA,\ell}(Q)|$ is small whenever $D(P,Q)$ is small.
	A weaker notion of continuity is semicontinuity:
	$H_{\sA,\ell}(P)$ is \emph{upper} (or \emph{lower}) \emph{semicontinuous} at $P=Q$ with respect to $D(\cdot,\cdot)$, if for any $\eps>0$, there exists a $\delta>0$ such that 
	\begin{align}\label{eq:def_semicont}
	H_{\sA,\ell}(P)-H_{\sA,\ell}(Q) < \eps \,\, \text{ (or $H_{\sA,\ell}(Q)-H_{\sA,\ell}(P) < \eps$) } 
	\end{align}
	for all $P$ satisfying $D(P, Q) < \delta$.
	There are other ways to define the continuity in distribution of the generalized entropy, e.g.\ the order of $P$ and $Q$ in $D(P,Q)$ in the above definitions can be changed, or the continuity can be defined in the sequential continuity manner, or defined in terms of the continuity of mappings between topological spaces. Since the statistical distances under consideration may not be real metrics, and can generate different topologies on the space of distributions, these definitions are generally not equivalent (c.f.~\cite{Harremoes2007} on a discussion of this issue for Shannon entropy). 
	Not attempting to draw connections among different notions of continuity in distribution, in this work we investigate the sufficient conditions on the action space $\sA$, the loss function $\ell$ and the distribution $Q$ to make $H_{\sA,\ell}(P)$ continuous or semicontinuous at $Q$ according to the definitions in \eqref{eq:def_cont} and \eqref{eq:def_semicont}.
	Specifically, given distributions $P$ and $Q$ on $\sZ$, we derive upper and lower bounds for $H_{\sA,\ell}(P)-H_{\sA,\ell} (Q)$ in terms of various statistical distances between $P$ and $Q$. 
	This is the objective of Section~\ref{sec:H_diff}.
	
	The main route to bounding the entropy difference taken in Section~\ref{sec:H_diff} is by relaxing the variational representation of the generalized entropy, which results in bounds in Sections~\ref{sec:H_diff_TV} to \ref{sec:H_diff_dlA}.
	Following this route, in Sections~\ref{sec:H_diff_TV}, \ref{sec:H_diff_KL} and \ref{sec:H_diff_Chi2}, we derive bounds for the entropy difference in terms of the total variation distance, KL divergence and $\chi^2$ divergence between $P$ and $Q$ on $\sZ$.
	Among the results in terms of the KL divergence, we show a connection between the Lipschitz continuity of the R\'enyi entropy in the entropy order and the continuity of the Shannon/differential entropy in the underlying distribution.
	These bounds are sharpened in Section~\ref{sec:H_diff_P_ell} by considering the distance between distributions of the loss under $P$ and $Q$ when an optimal action is taken.
	In Section~\ref{sec:H_diff_Wass}, we propose a general method to bound the entropy difference in terms of the Wasserstein distance, which depends on the property of the loss function.
	In Section~\ref{sec:H_diff_dlA}, we examine a bound in terms of a distance that depends on both the action space and the loss function.
	In Section~\ref{sec:H_diff_Breg}, we take a different route to show an exact representation of the entropy difference involving the Bregman divergence generated by the negative entropy, which is based on the concavity of the generalized entropy, and also induces bounds in terms of the Euclidean distance between the two distributions. 
	In Section~\ref{sec:H_diff_comp}, comparisons are made between the results derived in this work and the existing bounds on the entropy difference in the literature.
	Finally, an information-theoretic application of the results is presented in Section~\ref{sec:App_mi_ub}, where new upper bounds on the mutual information are derived using the new entropy difference bounds in terms of KL divergence and total variation distance. 
	The results in Section~\ref{sec:H_diff} have been presented in part in \cite{Xu_cont_H_ISIT20}.
	

	\subsection{Applications to statistical learning theory}
	While the continuity properties of the generalized entropy may find applications in a variety of subjects, in this work we focus on studying their applications to the theory of statistical learning.
	We show that the three major paradigms of statistical learning, namely the frequentist learning, the Bayesian learning, and learning by fitting the empirical distribution with a predefined family of distributions, all can be studied under the framework of the continuity of generalized entropy.
	
	In Section~\ref{sec:App}, we show that the excess risk of the ERM algorithm in the frequentist learning can be analyzed with the upper bounds on the entropy difference obtained in Section~\ref{sec:H_diff}, in terms of the statistical distance between the data-generating distribution and the empirical distribution.
	In particular, we give two examples where the success of the ERM algorithm does not directly depend on the hypothesis class, but on the underlying distribution and the loss function.
	We also reveal an intimate connection between a generalized notion of typicality in information theory and the learnability of a hypothesis class, through an entropy continuity argument.
	
	In Section~\ref{sec:MER_Bayes}, we give an overview of using the continuity property of the generalized entropy to analyze the minimum excess risk in Bayesian learning, which is studied in detail in \cite{MER19}.
	The main idea is to bound the entropy difference in terms of the statistical distance between the posterior predictive distribution and the true predictive model, which leads to upper bounds for the minimum excess risk in terms of the minimum estimation error of the model parameters.
	
	The study of the continuity of generalized entropy is extended to the generalized conditional entropy in Section~\ref{sec:App_mismatch}.
	Based on conditional entropy difference bounds, we derive upper bounds for the excess risk in Bayes decision-making problems with distributional mismatch. An application of the results is the excess risk analysis of a third paradigm of learning, where the learned decision rule is optimally designed under a surrogate of the data-generating distribution, which is found by projecting the empirical distribution to an exponential family of distributions.
	This method of analysis may also shed some light on the in-distribution excess risk analysis of the recently proposed maximum conditional entropy and minimax frameworks of statistical learning \cite{FarniaTse16,LR_minimax}.

	\subsection{Novelty}
	The continuity of Shannon entropy has been known for decades. A result regarding this property can be found in \cite[Lemma~2.7]{CsiKor_book1st} and \cite[Theorem~17.3.3]{Cover_book} 
	in terms of the total variation distance.
	In \cite{Zhang_entropy07}, a tighter such bound is derived via an optimal coupling argument, further improvement of which are given in \cite{Ho_Yeung_entropy10} and \cite{Sason_entropy13}.
	The continuity of differential entropy has been studied much more recently in \cite{PoliWu06} in terms of the Wasserstein distance.
	The results on Shannon/differential entropy obtained in this work have their own merits compared to the existing results, which will be discussed in Section~\ref{sec:H_diff_comp}. Beyond Shannon/differential entropy, in \cite{WuVerdu12_fmmseMI} the continuity of the MMSE $H_2(Z|X)$ in the joint distribution $P_{Z,X}$ and in the prior distribution $P_Z$ is investigated.
	For the generalized entropy defined in \eqref{eq:H_def} with general loss functions, as well as the generalized conditional entropy defined in \eqref{eq:cond_H_def}, there has been no dedicated study on their continuity properties so far to the author's knowledge.
	
	It is also new to view the excess risk analysis for the learning problems through the continuity of generalized entropy.
	Most existing works on the frequentist learning focus on the complexity analysis of the hypothesis space, instead of directly comparing the distance between the data-generating distribution and the empirical distribution. The latter method leads to a new result in Theorem~\ref{prop:excess_Lip} that does not depend on the hypothesis space.
	The performance of Bayesian learning under a generative model with respect to general loss functions is much less studied than the frequentist learning. 
	The analysis based on entropy continuity provides a unique way to relate the minimum achievable excess risk to the model uncertainty, as illustrated by \eqref{eq:MER01_logistic} for Bayesian logistic regression.
	The method of supervised learning by designing the decision rule under a surrogate of the data-generating distribution is also less studied in the literature.
	Corollary~\ref{co:excess_exp_TV} addresses a special case of this problem and explicitly shows that the excess risk consists of a fixed term of approximation error and a vanishing term of estimation error.
	
	This work would make a first effort to develop general methods of analysis for the continuity property of the generalized entropy, establish connections to statistical learning theory, and draw attention of researchers in related fields on its potentially broader applications.

	\section{Bounds on entropy difference}\label{sec:H_diff}
	In this section, we derive upper and lower bounds on the entropy difference between two distributions $P$ and $Q$ in terms of their total variation distance, KL divergence, $\chi^2$ divergence, Wasserstein distance, and a semidistance that depends on $\sA$ and $\ell$. We also compare the new results with existing ones, and apply some of the new results to derive new upper bounds for the mutual information.
	
	In what follows, we assume the infimum in \eqref{eq:H_def} can be achieved for all distributions, and let $a_P$ and $a_Q$ be the optimal actions achieving the infimum under distributions $P$ and $Q$ respectively. Then we have
	$H_\ell(P) = \E_P[\ell(Z,a_P)]$ and $H_\ell(Q) = \E_Q[\ell(Z,a_Q)]$.
	The results in Sections~\ref{sec:H_diff_TV} to \ref{sec:H_diff_dlA} build on the following lemma, a consequence of the definitions of $a_P$ and $a_Q$, and the variational representation of the generalized entropy in \eqref{eq:H_def}.
	\begin{lemma}\label{lm:H_diff_E}
		Suppose there exist actions $a_P$ and $a_Q$ in $\sA$ such that $H_\ell(P) = \E_P[\ell(Z,a_P)]$ and $H_\ell(Q) = \E_Q[\ell(Z,a_Q)]$, then
		\begin{align}\label{eq:H_diff_lbub}
			\E_P[\ell(Z,a_P)] - \E_Q[\ell(Z,a_P)] \le	H_\ell(P) - H_\ell(Q) 
			\le \E_P[\ell(Z,a_Q)] - \E_Q[\ell(Z,a_Q)] .
		\end{align}
	\end{lemma}

	\subsection{Bounds via total variation distance}\label{sec:H_diff_TV}
	\subsubsection{General results}
	We first show that when the loss function is uniformly bounded, the entropy difference can be controlled in terms of the total variation distance between the two distributions, defined as $d_{\rm TV}(P,Q)\deq \frac{1}{2}\int_{\sZ}|P-Q|(\rd z)$.
	\begin{theorem}\label{prop:cont_TV}
		If $\ell(\cdot,a_Q)\in [\alpha_Q,\beta_Q]$ for all $z\in\sZ$, then
		\begin{align}\label{eq:H_diff_TV_ub}
			H_\ell(P) - H_\ell(Q) \le (\beta_Q-\alpha_Q) d_{\rm TV}(P,Q) .
		\end{align}
		Consequently, if $\ell(\cdot,a_P)\in [\alpha_P,\beta_P]$ for all $z\in\sZ$, then
		\begin{align}\label{eq:H_diff_TV_lb}
			H_\ell(Q) - H_\ell(P) \le (\beta_P-\alpha_P) d_{\rm TV}(P,Q) .
		\end{align}
		
	\end{theorem}
	\begin{proof}
		The upper bound in \eqref{eq:H_diff_TV_ub} can be shown by
		\begin{align}
			H_\ell(P) - H_\ell(Q) &\le \E_P[\ell(Z,a_Q)] - \E_Q[\ell(Z,a_Q)] \\
			&=  \int_{\sZ} \ell(z,a_Q) (P-Q)({\rm d}z) \\
			&= \int_{\sZ} \big( \ell(z,a_Q) - (\alpha_Q+\beta_Q)/2 \big) (P-Q)({\rm d}z) \\
			&\le \int_{\sZ} \frac{\beta_Q-\alpha_Q}{2} |P-Q|({\rm d}z) \\
			&= (\beta_Q-\alpha_Q) d_{\rm TV}(P,Q) ,
		\end{align}
		where the first step follows from Lemma~\ref{lm:H_diff_E}, and the last step follows the definition of $d_{\rm TV}(P,Q)$.
		The upper bound in \eqref{eq:H_diff_TV_lb} follows by exchanging the roles of $P$ and $Q$, and the fact that $d_{\rm TV}(P,Q)=d_{\rm TV}(Q,P)$.
	\end{proof}

	\subsubsection{Examples}
	Applying Theorem~\ref{prop:cont_TV} to the log loss, we obtain new bounds for the Shannon/differential entropy.
	\begin{corollary}\label{co:log_TV}
		For both discrete and continuous $\sZ$, let $\bar P = \sup_{ z\in\sZ}P(z)/\inf_{ z\in\sZ}P(z)$ and $\bar Q = \sup_{ z\in\sZ}Q(z)/\inf_{ z\in\sZ}Q(z)$. Then
		\begin{align}
			H_{\log} (P) - H_{\log} (Q) \le \big(\log \bar Q\big) d_{\rm TV}(P,Q) ,
		\end{align}
		and 
		\begin{align}
			|H_{\log} (P) - H_{\log} (Q)| \le \big(\log(\bar{P} \vee \bar{Q}) \big) d_{\rm TV}(P,Q) . \label{eq:log_abs_TV}
		\end{align}
	\end{corollary}

	Next, applying Theorem~\ref{prop:cont_TV} to the quadratic loss, we obtain a bound for the variance difference between two distributions on a bounded interval in terms of their total variation distance.
	\begin{corollary}\label{co:2_TV}
		If $\sZ\subset[\alpha,\beta]\subset\R$, then
		\begin{align}
			\big|\Var_P[Z] - \Var_Q[Z]\big| \le (\beta-\alpha)^2 d_{\rm TV}(P,Q) .
		\end{align}
	\end{corollary}
	\begin{proof}
		From the assumption that $\sZ\subset[\alpha,\beta]$, we have that for any $z\in\sZ$, $0\le \ell(z,a_P)=(z-\E_P Z)^2\le (\beta-\alpha)^2$ and $0\le \ell(z,a_Q)=(z-\E_Q Z)^2\le (\beta-\alpha)^2$. The result then follows from Theorem~\ref{prop:cont_TV}.
	\end{proof}
	
	Additionally, applying Theorem~\ref{prop:cont_TV} to the zero-one loss, we immediately have the following result.
	\begin{corollary}\label{co:01_TV}
		If $\sZ$ is discrete, then
		\begin{align}
			\big|\max_{ z\in\sZ}P(z) - \max_{ z\in\sZ}Q(z)\big| \le d_{\rm TV}(P,Q) .
		\end{align}
	\end{corollary}
	
	\subsection{Bounds via KL divergence}\label{sec:H_diff_KL}
	\subsubsection{General results}
	The next set of results present sufficient conditions for the entropy difference to be controlled by the KL divergence between the two distributions.
	These results may apply to the generalized entropy with an unbounded loss function.
	Recall that a random variable $U$ is $\sigma^2$-subgaussian if $\E[e^{\lambda (U-\E U)}] \le e^{\lambda^2\sigma^2/2}$ for all $\lambda\in\R$. 
	\begin{theorem}\label{prop:cont_KL}
		If $\ell(Z,a_Q)$ is $
		\sigma_Q^2$-subgaussian under $Q$, then
		\begin{align}\label{eq:subG_cont_KL_Q}
			H_\ell(P) - H_\ell(Q) \le \sqrt{2\sigma_Q^2 D(P \| Q)} ;
		\end{align}
		for the other direction, if $\ell(Z,a_P)$ is $\sigma_P^2$-subgaussian under $Q$, then
		\begin{align}\label{eq:subG_cont_KL_P}
			H_\ell(Q) - H_\ell(P) \le \sqrt{2\sigma_P^2 D(P \| Q)} .
		\end{align}
		More generally, if there exists a function $\varphi_Q$ over $[0,b_Q)$ with some $b_Q \in (0, \infty]$ such that
		\begin{align}
			\log \E_Q\left[e^{\lambda \left(\ell(Z,a_Q) - \E_Q[\ell(Z,a_Q) ]\right)}\right] \le \varphi_Q(\lambda) \label{eq:DPQ_var_cond+_th}
		\end{align}
		for all $0\le \lambda < b_Q$, then
		\begin{align}\label{eq:DPQ_diff_+}
			H_\ell(P) - H_\ell(Q) \le \varphi_Q^{*-1}(D(P\| Q)) ;
		\end{align}
		for the other direction, if there exists a function $\varphi_P$ over $[0,b_P)$ with some $b_P\in (0, \infty]$ such that
		\begin{align}
			\log \E_Q\left[e^{-\lambda \left(\ell(Z,a_P) - \E_Q[\ell(Z,a_P)] \right)}\right] \le \varphi_P(\lambda) \quad \label{eq:DPQ_var_cond-_th}
		\end{align}
		for all $ 0\le \lambda < b_P$,
		then
		\begin{align}\label{eq:DPQ_diff_-}
			H_\ell(Q) - H_\ell(P) \le \varphi_P^{*-1}(D(P\| Q)) ;
		\end{align}
		where 
		$
		\varphi_Q^*(\gamma) \deq \sup_{0\le \lambda < b_Q} \lambda \gamma - \varphi_Q(\lambda) 
		$
		and
		$
		\varphi_P^*(\gamma) \deq \sup_{0\le \lambda < b_P} \lambda \gamma - \varphi_P(\lambda) ,
		$
		$\gamma\in\R$,
		are Legendre duals of $\varphi_Q$ and $\varphi_P$;
		and $\varphi_Q^{*-1}$ and $\varphi_P^{*-1}$ are the generalized inverses of $\varphi_Q^*$ and $\varphi_P^*$, defined as
		$
		\varphi_Q^{*-1}(x) \deq \sup\{\gamma\in\R:\varphi_Q^*(\gamma)\le x\} 
		$
		and
		$
		\varphi_P^{*-1}(x) \deq \sup\{\gamma\in\R:\varphi_P^*(\gamma)\le x\} ,  
		$
		$x\in\R $.
		In addition, if $\varphi_Q(\lambda)$ is strictly convex over $(0, b_Q)$ and $\varphi_Q(0) =  \varphi_Q'(0)=0$, then 
		$
		\lim_{x\downarrow 0}\varphi_Q^{*-1}(x) = 0 ;
		$
		similarly, if $\varphi_P(\lambda)$ is strictly convex over $(0, b_P)$ and $\varphi_P(0) =  \varphi_P'(0)=0$, then 
		$
		\lim_{x\downarrow 0}\varphi_P^{*-1}(x) = 0. 
		$
	\end{theorem}
	\noindent{\bf Remark.} By exchanging the roles of $P$ and $Q$ in Theorem~\ref{prop:cont_KL}, we can obtain another set of bounds for the entropy difference in terms of $D(Q\| P)$ under appropriate conditions.
	\begin{proof}[Proof of Theorem~\ref{prop:cont_KL}]
		The results in \eqref{eq:subG_cont_KL_Q} and \eqref{eq:subG_cont_KL_P} are special cases of the general results in \eqref{eq:DPQ_diff_+} and \eqref{eq:DPQ_diff_-} respectively, with $\varphi_Q(\lambda) = {\sigma_Q^2\lambda^2}/{2}$, $\varphi_P(\lambda)=\sigma_P^2\lambda^2/2$, and $b_Q=b_P=\infty$, such that $\varphi_Q^*(\gamma) = {\gamma^2}/{2\sigma_Q^2}$ and $\varphi_P^*(\gamma) = {\gamma^2}/{2\sigma_P^2}$.
		The general results are consequences of Lemma~\ref{lm:H_diff_E} and Lemma~\ref{lm:DPQ_variational_gen} stated below, instantiated with $f(z) = \ell(z,a_Q)$, $\varphi_+(\lambda)=\varphi_Q(\lambda)$ and $b_+ = b_Q$ for \eqref{eq:DPQ_diff_+}, and with $f(z) = \ell(z,a_P)$, $\varphi_-(\lambda)=\varphi_P(\lambda)$ and $b_- = b_P$ for \eqref{eq:DPQ_diff_-}.
	\end{proof}
	
	\begin{lemma}\label{lm:DPQ_variational_gen}
		For distributions $P$ and $Q$ on an arbitrary set $\sZ$ and a function $f:\sZ\rightarrow\R$, if there exists a function $\varphi_+$ over $[0,b_+)$ with some $b_+ \in (0, \infty]$ such that
		\begin{align}
			\log \E_Q\left[e^{\lambda \left(f(Z) - \E_Q f(Z) \right)}\right] \le \varphi_+(\lambda), \quad\forall\,  0\le \lambda < b_+ , \label{eq:DPQ_var_cond+}
		\end{align}
		then
		\begin{align}\label{eq:DPQ_variational_gen_+}
			\E_P[f(Z)] - \E_Q[f(Z)] \le \varphi_+^{*-1}(D(P\| Q)) ;
		\end{align}
		for the other direction, if there exists a function $\varphi_-$ over $[0,b_-)$ with some $b_- \in (0, \infty]$ such that
		\begin{align}
			\log \E_Q\Big[e^{-\lambda \left(f(Z) - \E_Q f(Z) \right)}\Big] \le \varphi_-(\lambda), \quad\forall\, 0\le \lambda < b_-, \label{eq:DPQ_var_cond-}
		\end{align}
		then
		\begin{align}\label{eq:DPQ_variational_gen_-}
			\E_Q[f(Z)] - \E_P[f(Z)] \le \varphi_-^{*-1}(D(P\| Q)) ;
		\end{align}
		where 
		\begin{align}
			\varphi_+^*(\gamma) &\deq \sup_{0\le \lambda \le b_+} \lambda \gamma - \varphi_+(\lambda) , \quad \gamma\in\R \\
			\varphi_-^*(\gamma) &\deq \sup_{0\le \lambda \le b_-} \lambda \gamma - \varphi_-(\lambda) , \quad \gamma\in\R 
		\end{align}
		are Legendre duals of $\varphi_+$ and $\varphi_-$,
		and $\varphi_+^{*-1}$ and $\varphi_-^{*-1}$ are the generalized inverses of $\varphi_+^*$ and $\varphi_-^*$,
		\begin{align}
			\varphi_+^{*-1}(x) &\deq \sup\{\gamma\in\R:\varphi_+^*(\gamma)< x\} , \quad x\in\R \label{eq:inv_Legendre_dual} \\
			\varphi_-^{*-1}(x) &\deq \sup\{\gamma\in\R:\varphi_-^*(\gamma)< x\} , \quad x\in\R .
		\end{align}
		In addition, if $\varphi_+(\lambda)$ is strictly convex over $(0, b_+)$ and $\varphi_+(0) =  \varphi_+'(0)=0$, then 
		\begin{align}
			\lim_{x\downarrow 0}\varphi_+^{*-1}(x) = 0 ;
		\end{align}
		similarly, if $\varphi_-(\lambda)$ is strictly convex over $(0, b_-)$ and $\varphi_-(0) =  \varphi_-'(0)=0$, then 
		\begin{align}
			\lim_{x\downarrow 0}\varphi_-^{*-1}(x) = 0. 
		\end{align}
	\end{lemma}
	As a concrete example of Lemma~\ref{lm:DPQ_variational_gen}, if $f(Z)$ is $\sigma^2$-subgaussian under $Q$, then choosing $\varphi_+(\lambda)=\varphi_-(\lambda)=\sigma^2\lambda^2/2$ and $b_+ = b_- = \infty$ leads to the well-known bound 
	\begin{align}\label{eq:E_diff_KL_subG}
		|\E_P f(Z) - \E_Q f(Z)|\le \sqrt{2\sigma^2 D(P \| Q)} ,
	\end{align}
	which is used in proving \eqref{eq:subG_cont_KL_Q} and \eqref{eq:subG_cont_KL_P}.
	
	Lemma~\ref{lm:DPQ_variational_gen} is proved in Appendix~\ref{appd:pf_DPQ_variational_gen}. The proof is adapted from \cite[Lemma~4.18]{concen_ineq_BLM}, \cite[Theorem~2]{JiaoHanWeissman_bias_17} and \cite[Theorem~1]{BuZouVvv19}.
	It is worthwhile to point out that by properly defining the inverse functions $\varphi_+^{*-1}$ and $\varphi_-^{*-1}$, the restrictions on the functions $\varphi_+$ and $\varphi_-$ in terms of convexity and boundary conditions $\varphi_\pm(0)=\varphi_\pm'(0)=0$ imposed in the references are not needed to prove \eqref{eq:DPQ_variational_gen_+} and \eqref{eq:DPQ_variational_gen_-}. However, with these conditions we can show that $\lim_{x\downarrow 0}\varphi_+^{*-1}(x) = 0$ and $\lim_{x\downarrow 0}\varphi_-^{*-1}(x) = 0$, which is needed by Theorem~\ref{prop:cont_KL} for proving the continuity of the generalized entropy.

	\subsubsection{Example: variance comparison against Gaussian}
	As the first application of the general results in Theorem~\ref{prop:cont_KL}, we consider bounding the variance difference between an arbitrary real-valued random variable, potentially unbounded, and a Gaussian random variable.
	\begin{corollary}\label{co:2_Gaussian_KL}
		For the quadratic loss, if $Z$ is Gaussian with variance $\sigma^2$ and an arbitrary mean under $Q$, then for any $P$ on $\R$,
		\begin{align}
			\big|\Var_P[Z] - \Var_Q[Z]\big| &\le 2\sigma^2 \Big( \sqrt{D(P \| Q)} + D(P \| Q) \Big) . \label{eq:var_diff_G_abs} 
		\end{align}
	\end{corollary}
	\begin{proof}
		We first prove that
		\begin{align}
			\Var_P[Z] - \Var_Q[Z] &\le 2\sigma^2 \Big( \sqrt{D(P \| Q)} + D(P \| Q) \Big) \label{eq:var_diff_G_P} .
		\end{align}
		Under $Q$, $(Z-\E_Q Z)^2$ has the same distribution as $\sigma^2 U^2$, where $U$ is standard Gaussian.
		From the moment generating function of the $\chi^2$ random variable, we have
		\begin{align}
			\log \E_Q\left[e^{\lambda \left((Z-\E_Q Z)^2 - \sigma^2 \right)}\right] = - \frac{1}{2}\log(1-2\sigma^2\lambda)  -\sigma^2\lambda, \quad  -\infty<\lambda<\frac{1}{2\sigma^2} .
		\end{align}
		It can be verified that \eqref{eq:DPQ_var_cond+_th} in Theorem~\ref{prop:cont_KL} is satisfied with
		$\varphi_Q(\lambda)=\sigma^4\lambda^2/(1-2\sigma^2\lambda)$ and $b_Q={1}/{2\sigma^2}$ \cite[Section~2.4]{concen_ineq_BLM}, i.e.,
		\begin{align}
			\log \E_Q\left[e^{\lambda \left((Z-\E_Q Z)^2 - \sigma^2 \right)}\right] < \frac{\sigma^4 \lambda^2}{1-2\sigma^2\lambda}, \quad  \forall \,\, 0<\lambda<\frac{1}{2\sigma^2} .
		\end{align}
		Further, we have $\varphi_Q^*(\gamma)=(\sqrt{2\gamma+\sigma^2}-\sigma)^2/4\sigma^2$ and $\varphi_Q^{*-1}(x)=2\sigma^2(\sqrt{x}+x)$, which leads to \eqref{eq:var_diff_G_P} by \eqref{eq:DPQ_diff_+} in Theorem~\ref{prop:cont_KL}.
		
		Next, we prove the other direction
		\begin{align}
			\Var_Q[Z] - \Var_P[Z] &\le 2\sigma^2 \Big( \sqrt{D(P \| Q)} + D(P \| Q) \Big) . \label{eq:var_diff_G_Q}
		\end{align}
		Under $Q$, $(Z-\E_P Z)^2$ has the same distribution as $\sigma^2 U^2$, where $U$ is Gaussian with mean $(\E_Q[Z]-\E_P[Z])/\sigma$ and variance $1$.
		From the moment generating function of the non-central $\chi^2$ random variable, we have
		\begin{align}
			\log \E_Q\left[e^{-\lambda \left((Z-\E_P[Z])^2 - \E_Q[(Z-\E_P Z)^2] \right)}\right] = & -\frac{1}{2}\log (1+2\sigma^2\lambda) + \lambda \E_Q[(Z-\E_P[Z])^2] \nonumber \\
			& - \frac{(\E_Q[Z]-\E_P[Z])^2\lambda}{1+2\sigma^2\lambda} , \quad -\frac{1}{2\sigma^2}<\lambda<\infty .
		\end{align}
		Dropping the last term when $\lambda>0$, we have
		\begin{align}\label{eq:mgf_noncentral_chi2}
			\log \E_Q\left[e^{-\lambda \left((Z-\E_P[Z])^2 - \E_Q[(Z-\E_P Z)^2] \right)}\right] \le -\frac{1}{2}\log (1+2\sigma^2\lambda) + \lambda \E_Q[(Z-\E_P[Z])^2] , \quad \forall \lambda>0 .
		\end{align}
		It can be verified via Taylor expansion of the right-hand side of \eqref{eq:mgf_noncentral_chi2} that \eqref{eq:DPQ_var_cond-_th} in Theorem~\ref{prop:cont_KL} is satisfied with
		$\varphi_P(\lambda)=\sigma^4\lambda^2 - \big(\sigma^2-\E_Q[(Z-\E_P Z)^2]\big)\lambda$ and $b_P = \infty$, i.e.,
		\begin{align}
			\log \E_Q\left[e^{-\lambda \left((Z-\E_P[Z])^2 - \E_Q[(Z-\E_P Z)^2] \right)}\right] \le 
			\sigma^4\lambda^2 - \big(\sigma^2-\E_Q[(Z-\E_P Z)^2]\big)\lambda , \quad \forall \lambda>0 .
		\end{align}
		Further, we have $\varphi_P^*(\gamma)=(\gamma+\sigma^2 - \E_Q[(Z-\E_P Z)^2])^2/4\sigma^4$ and $\varphi_P^{*-1}(x)=2\sigma^2\sqrt{x}+(\E_P[Z] - \E_Q[Z])^2$, which leads to 
		\begin{align}
			\Var_Q[Z] - \Var_P[Z] &\le 2\sigma^2 \sqrt{D(P \| Q)} + \big(\E_P[Z] - \E_Q[Z]\big)^2 
		\end{align}
		by \eqref{eq:DPQ_diff_-} in Theorem~\ref{prop:cont_KL}.
		The upper bound in \eqref{eq:var_diff_G_Q} then follows from the fact that $(\E_P[Z] - \E_Q[Z])^2 \le 2\sigma^2 D(P \| Q)$, which is in turn due to the fact that $Z$ is Gaussian with variance $\sigma^2$ under $Q$ and \eqref{eq:E_diff_KL_subG} as a consequence of Lemma~\ref{lm:DPQ_variational_gen}.
	\end{proof}

	\subsubsection{Example: bounded loss functions}
	Next, we apply Theorem~\ref{prop:cont_KL} to the cases where the loss function is bounded. 
	Using the fact that a bounded random variable taking values in $[\alpha,\beta]$ is $(\beta-\alpha)^2/4$-subgaussian under any distribution, Theorem~\ref{prop:cont_KL} leads to the following corollary.
	\begin{corollary}\label{co:cont_KL_bdd}
		If $\ell(\cdot,a_Q)\in [\alpha_Q,\beta_Q]$ for all $z\in\sZ$, then
		\begin{align}\label{eq:bdd_cont_KL_Q}
			H_\ell(P) - H_\ell(Q) \le (\beta_Q-\alpha_Q)\sqrt{\frac{1}{2} D(P \| Q)} ;
		\end{align}
		if $\ell(\cdot,a_P)\in [\alpha_P,\beta_P]$ for all $z\in\sZ$, then
		\begin{align}\label{eq:bdd_cont_KL_P}
			H_\ell(Q) - H_\ell(P) \le (\beta_P-\alpha_P)\sqrt{\frac{1}{2} D(P \| Q)} .
		\end{align}
		In particular, for the log loss, using the notation in Corollary~\ref{co:log_TV}, 
		\begin{align}
			|H_{\log} (P) - H_{\log} (Q)| \le \big(\log(\bar{P} \vee \bar{Q}) \big) \sqrt{\frac{1}{2} D(P \| Q)}  ; \label{eq:log_abs_KL} 
		\end{align}
		for the quadratic loss, if $\sZ\subset[\alpha,\beta]\subset\R$, then
		\begin{align}
			\big|\Var_P[Z] - \Var_Q[Z]\big| \le (\beta-\alpha)^2 \sqrt{\frac{1}{2} D(P \| Q)} ;
		\end{align}
		while for the zero-one loss,
		\begin{align}\label{eq:H01_KL}
			|H_{01}(P) - H_{01}(Q)| \le \sqrt{\frac{1}{2}D(P\|Q)} .
		\end{align}
	\end{corollary}
	The results in Corollary~\ref{co:cont_KL_bdd} can also be derived from Theorem~\ref{prop:cont_TV}, Corollary~\ref{co:log_TV}, \ref{co:2_TV}, and \ref{co:01_TV} respectively, via Pinsker's inequality \cite{fdiv_Sas_Ver}.

	\subsubsection{Example: subgaussian log loss and connection to R\'enyi entropy order}
	For the log loss, Theorem~\ref{prop:cont_KL} also provide bounds for the case where $\ell(\cdot,a_Q)$ and $\ell(\cdot,a_P)$ are unbounded but subgaussian, as stated in Corollary~\ref{co:cont_KL_Renyi} below. The results reveal a connection between the continuity of the Shannon/differential entropy in distribution and the deviation of the R\'enyi (cross) entropy from the ordinary (cross) entropy. We define the \emph{R\'enyi cross entropy} as follows.
	\begin{definition}
		For distributions $P$ and $Q$ on $\sZ$, the R\'enyi cross entropy between $Q$ and $P$ of order $\alpha$, where $\alpha\in\R\setminus\{1\}$, is defined as
		\begin{align}
			R_\alpha( Q, P) \deq \frac{1}{1-\alpha} \log \int_{\sZ} Q(\rd z) P(z)^{\alpha - 1} .
		\end{align}
	\end{definition}
	\noindent 
	
	Using L'Hôspital's rule, it can be shown that 
	$
	\lim_{\alpha\rightarrow 1}R_{\alpha}(Q,P) = R_1(Q,P) \deq - \int_{\sZ}Q(\rd z)\log P(z) ,
	$
	which is the ordinary cross entropy between $Q$ and $P$. When $P=Q$, $R_\alpha(Q, Q)$ can be written as
	\begin{align}
		R_\alpha(Q) \deq \frac{1}{1-\alpha} \log \int_{\sZ} Q(\rd z) Q(z)^{\alpha - 1} , \quad\alpha\neq 1,
	\end{align}
	which is the {R\'enyi entropy} of order $\alpha$ of $Q$; and
	$
	\lim_{\alpha\rightarrow 1} R_\alpha(Q) = R_1(Q) \deq H_{\log}(Q)
	$
	is the ordinary entropy of $Q$, which is the Shannon entropy if $\sZ$ is discrete and the differential entropy if $\sZ$ is continuous.
	Note that with the above definitions, $\alpha$ can take any value in $\R$, so that $R_\alpha(Q,P)$ and $R_\alpha(Q)$ can be either positive or negative. 
	\begin{corollary}\label{co:cont_KL_Renyi}
		For the log loss, if there exists $\sigma_Q>0$ such that
		$
		R_{1-\lambda} (Q) - R_1(Q)  \le {\lambda \sigma_Q^2 }/{2} 
		$
		for all $\lambda>0$, 
		then
		\begin{align}\label{eq:Renyi_ub_KL}
			H_{\log}(P) - H_{\log}(Q) \le \sqrt{ 2\sigma_Q^2 D(P \| Q) } .
		\end{align}
		For the other direction, if there exists $\sigma_P>0$ such that 
		$
		R_{1} (Q,P) - R_{1+\lambda}(Q,P)  \le {\lambda \sigma_P^2 }/{2} 
		$
		for all $\lambda>0$, 
		then
		\begin{align}
			H_{\log}(Q) - H_{\log}(P) \le \sqrt{ 2\sigma_P^2 D(P \| Q) } .
		\end{align}
	\end{corollary}
	\begin{proof}
		To prove the first upper bound, note that
		\begin{align}
			\log \E_Q\big[e^{\lambda (- \log Q(Z) - \E_Q[-\log Q(Z)])}\big] = \lambda (R_{1-\lambda} (Q) - R_1(Q)) .
		\end{align}
		If $R_{1-\lambda} (Q) - R_1(Q)  \le {\lambda \sigma_Q^2 }/{2}$ for all $\lambda>0$, 
		then we can make use of \eqref{eq:DPQ_diff_+} in Theorem~\ref{prop:cont_KL} with $\varphi_Q(\lambda) = {\lambda^2 \sigma_Q^2 }/{2}$,
		and get
		\begin{align}
			H_{\log}(P) - H_{\log}(Q) \le \sqrt{ 2\sigma_Q^2 D(P \| Q) } .
		\end{align}
		
		Similarly, for the second upper bound, note that
		\begin{align}
			\log \E_Q\big[e^{-\lambda (- \log P(Z) - \E_Q[-\log P(Z)])}\big] = \lambda (R_{1} (Q,P) - R_{1+\lambda}(Q,P)) .
		\end{align}
		If $R_{1} (Q,P) - R_{1+\lambda}(Q,P)  \le {\lambda \sigma_P^2 }/{2}$ for all $\lambda>0$, 
		then we can make use of \eqref{eq:DPQ_diff_-} in Theorem~\ref{prop:cont_KL} with $\varphi_P(\lambda) = {\lambda^2 \sigma_P^2 }/{2}$,
		and get
		\begin{align}
			H_{\log}(Q) - H_{\log}(P) \le \sqrt{ 2\sigma_P^2 D(P \| Q) } .
		\end{align}
	\end{proof}
	The upper bound in \eqref{eq:Renyi_ub_KL} of Corollary~\ref{co:cont_KL_Renyi} essentially states that if the R\'enyi entropy of a distribution is Lipschitz continuous in the entropy order at order 1, then the Shannon/differential entropy is upper-semicontinuous at that distribution.
	Further, if both $\sigma_Q$ and $\sigma_{P}$ in Corollary~\ref{co:cont_KL_Renyi} are upper-bounded by some $\beta > 0$ for all $P$ within a small neighborhood of $Q$ in terms of KL divergence, then it implies that the Shannon/differential entropy is continuous at $Q$.

	\subsection{Bounds via $\chi^2$ divergence}\label{sec:H_diff_Chi2}
	\subsubsection{General results}
	To further investigate the conditions for the generalized entropy with unbounded loss functions to be continuous, we consider the continuity in terms of the $\chi^2$ divergence, defined as $\chi^2(P\| Q)\deq\E_Q[(\frac{{\rm d}P}{{\rm d}Q}-1)^2]$.
	\begin{theorem}\label{prop:cont_Chi2}
		For distributions $P$ and $Q$ on $\sZ$, if $\Var_Q[\ell(Z,a_Q)]$ and $\Var_Q[\ell(Z,a_P)]$ exist, then
		\begin{align}\label{eq:H_diff_ub_chi2}
			H_\ell(P) - H_\ell(Q) 
			&\le \sqrt{ \Var_Q[\ell(Z,a_Q)] \chi^2(P \| Q)} ,
		\end{align}
		and
		\begin{align}\label{eq:H_diff_lb_chi2}
			H_\ell(Q) - H_\ell(P) 
			&\le \sqrt{ \Var_Q[\ell(Z,a_P)] \chi^2(P \| Q)} .
		\end{align}
	\end{theorem}
	\noindent{\bf Remark.} By exchanging the roles of $P$ and $Q$ in Theorem~\ref{prop:cont_Chi2}, we can obtain another set of bounds for the entropy difference in terms of $\chi^2(Q\| P)$ under appropriate conditions.
	\begin{proof}[Proof of Theorem~\ref{prop:cont_Chi2}]
		The proof is based on the Hammersley-Chapman-Robbins (HCR) lower bound for $\chi^2$ divergence \cite{Wu_lec_stat}, which states that for any distributions $P_U$ and $Q_U$ on a set $\sU$,
		\begin{align}
			\chi^2(P_U \| Q_U) \ge \frac{(\E[P_U] - \E[Q_U])^2}{\Var[Q_U]} . 
		\end{align}
		Applying the HCR lower bound to $\ell(Z,a_Q)$ and $\ell(Z,a_P)$ in the upper and lower bound in Lemma~\ref{lm:H_diff_E} respectively, and using the data processing inequality for $\chi^2$ divergence, we obtain the bounds in \eqref{eq:H_diff_ub_chi2} and \eqref{eq:H_diff_lb_chi2}.
	\end{proof}
	The upper bound in \eqref{eq:H_diff_ub_chi2} of Theorem~\ref{prop:cont_Chi2} implies that the generalized entropy is upper semicontinuous at $Q$ in terms of $\chi^2$ divergence, as long as $\Var_Q[\ell(Z,a_Q)]$ is finite.
	Further, if $\Var_Q[\ell(Z,a_P)]$ is upper-bounded by some $\beta > 0$ for all $P$ within a small neighborhood of $Q$ in terms of $\chi^2$ divergence, then Theorem~\ref{prop:cont_Chi2} implies that the generalized entropy is continuous at $Q$.
	Compared with the conditions for continuity in terms of total variation distance and KL divergence as stated in Theorem~\ref{prop:cont_TV} and Theorem~\ref{prop:cont_KL}, continuity of the generalized entropy in terms of $\chi^2$ divergence requires minimal conditions on $\ell$ and $Q$ as shown in Theorem~\ref{prop:cont_Chi2}.

	\subsubsection{Examples}
	Applying Theorem~\ref{prop:cont_Chi2} to the log loss, we get the following results for Shannon/differential entropy.
	\begin{corollary}\label{co:log_chi2}
		For distributions $P$ and $Q$ on $\sZ$, we have
		\begin{align}\label{eq:H_ub_logChi2}
			H_{\log}(P) - H_{\log}(Q) \le \sqrt{ \Var_Q[\log {Q(Z)}] \chi^2(P \| Q)} ,
		\end{align}
		where $\Var_Q[\log {Q(Z)}]$ is known as the varentropy of distribution $Q$ \cite{Varentropy_Kon_Ver}.
		Moreover,
		\begin{align}\label{eq:H_lb_logChi2}
			H_{\log}(Q) - H_{\log}(P) \le \sqrt{ \Var_Q[\log {P(Z)}] \chi^2(P \| Q)} ,
		\end{align}
		where $\Var_Q[\log {P(Z)}]$ may be called the cross varentropy of distribution $P$ under distribution $Q$.
	\end{corollary}
	
	Applying Theorem~\ref{prop:cont_Chi2} to the quadratic loss, we can deduce the following bounds on the variance difference.
	\begin{corollary}\label{co:var_chi2}
		For distributions $P$ and $Q$ on $\sZ\subset\R$, we have
		\begin{align}\label{eq:Var_ub_Chi2}
			\Var_P[Z] - \Var_Q[Z] \le \sqrt{\Var_Q\big[(Z-\E_Q [Z])^2\big] \chi^2(P\|Q)} ,
		\end{align}
		and
		\begin{align}\label{eq:Var_lb_Chi2}
			\Var_Q[Z] - \Var_P[Z] \le \sqrt{\Var_Q\big[(Z-\E_P [Z])^2\big] \chi^2(P\|Q)} .
		\end{align}
	\end{corollary}
	\noindent Compared with Corollary~\ref{co:log_TV} and Corollary~\ref{co:2_TV}, we see that the results in Corollary~\ref{co:log_chi2} and Corollary~\ref{co:var_chi2} do not require $Z$ or its log probability to take values in a bounded interval.

	\subsection{Bounds via $D(P_\ell,Q_\ell)$}\label{sec:H_diff_P_ell}
	We have derived bounds for the entropy difference in terms of several $f$-divergences between distributions $P$ and $Q$ on $\sZ$, which lead to sufficient conditions on the entropy continuity.
	If our purpose is merely bounding the entropy difference rather than examining its dependence on certain statistical distance $D(P,Q)$, we may bound it in terms of the distributional change of the loss when an optimal action is taken, e.g. either $\ell(Z,a_P)$ or $\ell(Z,a_Q)$, when the distribution of $Z$ changes from $P$ to $Q$.
	In other words, we can examine the statistical distance between $P_{\ell(Z,a_Q)}$ and $Q_{\ell(Z,a_Q)}$, or between $P_{\ell(Z,a_P)}$ and $Q_{\ell(Z,a_P)}$.
	The following result is a consequence of Lemma~\ref{lm:H_diff_E} and the proof techniques used in the previous subsections.
	\begin{theorem}\label{th:D_PQ_ell}
		For all the results derived in Sections~\ref{sec:H_diff_TV}, \ref{sec:H_diff_KL} and \ref{sec:H_diff_Chi2}, the upper bounds for $H_\ell(P)-H_\ell(Q)$ continue to hold when the corresponding statistical distance $D(P,Q)$ is replaced by $D(P_{\ell(Z,a_Q)} , Q_{\ell(Z,a_Q)})$; 
		and the upper bounds for $H_\ell(Q)-H_\ell(P)$ continue to hold when $D(P,Q)$ is replaced by $D(P_{\ell(Z,a_P)} , Q_{\ell(Z,a_P)})$.
	\end{theorem}
	\noindent 
	
	Due to the data processing inequality of the $f$-divergence, the bounds described in Theorem~\ref{th:D_PQ_ell} are tighter than their counterparts in the previous sections.
	To illustrate the potential improvement, we examine a case where $\sZ=\R^p$, $\sA = \{a\in\R^p: \|a\|=1\}$, and $\ell(z,a) = - a^\top z$. 
	Let the distributions $P$ and $Q$ on $\sZ$ be $\mathcal N(\mu_P,\sigma_P^2\mathbf{I})$ and $\mathcal N(\mu_Q,\sigma_Q^2\mathbf{I})$, with mean vectors $\mu_P,\mu_Q\in\R^p$ and elementwise variances $\sigma_P^2$ and $\sigma_Q^2$.
	Then, $H_{\ell}(P)=-\|\mu_P\|$ and $H_{\ell}(Q) = -\|\mu_Q\|$, with $a_P = \mu_P/\|\mu_P\|$ and $a_Q = \mu_Q/\|\mu_Q\|$.
	In addition, under $P$, $\ell(Z,a_P) \sim \mathcal{N}(-\|\mu_P\|,\sigma_P^2)$ and $\ell(Z,a_Q) \sim \mathcal{N}(-\mu_Q^\top \mu_P/\|\mu_Q\|,\sigma_P^2)$; while under $Q$, $\ell(Z,a_P) \sim \mathcal{N}(-\mu_P^\top \mu_Q/\|\mu_P\|,\sigma_Q^2)$ and $\ell(Z,a_Q) \sim \mathcal{N}(-\|\mu_Q\|,\sigma_Q^2)$.
	Applying Theorem~\ref{th:D_PQ_ell} to \eqref{eq:subG_cont_KL_Q} and \eqref{eq:subG_cont_KL_P}, respectively, in Theorem~\ref{prop:cont_KL} yields
	\begin{align}
		H_{\ell}(P) - H_{\ell}(Q) 
		\le \sqrt{ \Big( \|\mu_Q\| - \frac{\mu_Q^\top \mu_P}{\|\mu_Q\|} \Big)^2 +  \sigma_Q^2\Big(\frac{\sigma_P^2}{\sigma_Q^2} - 1 - \log\frac{\sigma_P^2}{\sigma_Q^2} \Big)} ,
	\end{align}
	and
	\begin{align}
		H_{\ell}(Q) - H_{\ell}(P) 
		\le \sqrt{ \Big ( \|\mu_P\| - \frac{\mu_P^\top \mu_Q}{\|\mu_P\|} \Big )^2  +  \sigma_Q^2 \Big(\frac{\sigma_P^2}{\sigma_Q^2} - 1 - \log\frac{\sigma_P^2}{\sigma_Q^2} \Big ) } ,
	\end{align}
	where the upper bounds do not depend on the dimension $p$ of $\sZ$.
	On the contrary, directly applying Theorem~\ref{prop:cont_KL} yields
	\begin{align}
		|H_{\ell}(Q) - H_{\ell}(P)| \le \sqrt{ \|\mu_P - \mu_Q \|^2 + p\sigma_Q^2\Big(\frac{\sigma_P^2}{\sigma_Q^2} - 1 - \log\frac{\sigma_P^2}{\sigma_Q^2} \Big)} ,
	\end{align}
	where the upper bound scales in $p$ as $O(\sqrt{p})$.
	This example shows that by considering the distributional change of the loss, Theorem~\ref{th:D_PQ_ell} can provide much tighter bounds on the entropy difference than the results obtained in the previous subsections.

	\subsection{Bounds via Wasserstein distance}\label{sec:H_diff_Wass}
	Another way to incorporate the loss function to the statistical distance between $P$ and $Q$ on $\sZ$ is by constructing a Wasserstein distance according to the property of $\ell$.
	We propose a general method to bound the entropy difference in terms of the Wasserstein distance.
	Suppose $\sZ$ is a metric space with some metric $d:\sZ\times\sZ\rightarrow\R_+$, then a Wasserstein distance ${\mathcal W}_d$ with respect to $d$ can be defined for distributions on $\sZ$ as
	\begin{align}
		{\mathcal W}_d(P,Q) \deq \inf_{P_{U,V}\in\Pi (P,Q)} \E[d(U,V)] ,
	\end{align}
	where $\Pi$ is the set of joint distributions on $\sZ\times\sZ$ with marginal distributions $P$ and $Q$.
	One can also define the Wasserstein distance with respect to $d$ of order $q$, with $q\in[1,\infty)$, as
	$
	\mathcal W_{d,q}(P,Q) \deq \inf_{P_{U,V}\in\Pi (P,Q)} \E[d(U,V)^q]^{1/q} .
	$
	A useful property of the Wasserstein distance is the Kantorovich-Rubinstein duality,
	\begin{align}\label{eq:W_dual}
		{\mathcal W}_d(P,Q) = \sup_{f:\sZ\rightarrow\R,\, \|f\|_{\rm Lip}\le 1}  (\E_P f - \E_Q f) ,
	\end{align}
	where $\|f\|_{\rm Lip}$ is the minimum value of $\alpha$ such that $|f(z)-f(z')|\le \alpha d(z,z')$ for all $z,z'\in\sZ$.
	Under the assumption that the loss function $\ell(\cdot,a)$ is Lipschitz in $z\in\sZ$ with respect to $d$ for all $a\in\sA$, \eqref{eq:W_dual} can be invoked to show the following bound on entropy difference.
	\begin{theorem}\label{prop:cont_Wass_dual}
		Suppose $\sZ$ is a metric space with metric $d$. If $\ell(\cdot,a_Q)$ is $\rho_Q$-Lipschitz in $z\in\sZ$ with respect to $d$, i.e. $|\ell(z,a_Q) - \ell(z',a_Q)| \le \rho_Q d(z,z')$ for all $z,z'\in\sZ$, then
		\begin{align}
			H_\ell(P) - H_\ell(Q) &\le \rho_Q {\mathcal W}_d(P,Q) ;
		\end{align}
		for the other direction, if $\ell(\cdot,a_P)$ is $\rho_P$-Lipschitz in $z\in\sZ$ with respect to $d$, then
		\begin{align}
			H_\ell(Q) - H_\ell(P) &\le \rho_P {\mathcal W}_d(P,Q) .
		\end{align}
	\end{theorem}
	\begin{proof}
		For one direction,
		\begin{align}
			H_\ell(P) - H_\ell(Q) 
			&\le \E_P[\ell(Z,a_Q)] -  \E_Q[\ell(Z,a_Q)] \\
			&\le  \rho_Q \sup_{f:\sZ\rightarrow\R,\, \|f\|_{\rm Lip}\le 1} (\E_P f - \E_Q f) \\
			&= \rho_Q {\mathcal W}_d(P,Q) ,
		\end{align}
		where the second inequality is due to the assumption that $\ell(\cdot,a_Q)$ is $\rho_Q$-Lipschitz in $z\in\sZ$; and the last step is due to the Kantorovich-Rubinstein duality of Wasserstein distance \eqref{eq:W_dual}.
		The other direction can be proved by exchanging the roles of $P$ and $Q$ and noting that ${\mathcal W}_d(P,Q)={\mathcal W}_d(Q,P)$.
	\end{proof}
	As a special case, when $\sZ = \sA$ and $\ell(\cdot,\cdot)$ is a metric on $\sZ$, then $\ell(\cdot,a)$ is $1$-Lipschitz in $z$ for all $a$ due to the triangle inequality, and we have the following particularly simple-looking bound.
	\begin{corollary}\label{prop:cont_Wass}
		If $\sZ=\sA$ is a metric space with metric $\ell(\cdot,\cdot)$, then
		\begin{align}
			|H_\ell(P) - H_\ell(Q)| &\le \mathcal W_\ell(P,Q) .
		\end{align}
	\end{corollary}
	\noindent 
	For example, for the zero-one loss, $\mathcal W_{01}(P,Q)=d_{\rm TV}(P,Q)$. Corollary~\ref{prop:cont_Wass} then implies that
	\begin{align}
		|H_{01}(P) - H_{01}(Q)| \le d_{\rm TV}(P,Q),
	\end{align}
	which is the same as the upper bound in Corollary~\ref{co:01_TV}.
	As another example, on the Euclidean space we have the following result.
	\begin{corollary}\label{co:Euc_Wass}
		If $\sZ=\sA=\R^p$ and $\ell(z,a)=\|z-a\|$ is the Euclidean distance on $\R^p$, then Corollary~\ref{prop:cont_Wass} implies that
		\begin{align}
			| H_{\text{\tiny $\|\!\cdot\!\|$}}(P) - H_{\text{\tiny $\|\!\cdot\!\|$}}(Q) | \le \mathcal W_{\text{\tiny $\|\!\cdot\!\|$}}(P,Q) . 
		\end{align}
	\end{corollary}
	\noindent In particular, for $p=1$, Corollary~\ref{co:Euc_Wass} implies that the difference between the minimum mean absolute deviation under $P$ and $Q$ is upper-bounded by the Wasserstein distance between $P$ and $Q$ with respect to the absolute difference.
	
	In addition, in view of Theorem~\ref{th:D_PQ_ell}, we have the following bounds for the entropy difference in terms of the Wasserstein distance between distributions of the loss.
	\begin{theorem}\label{prop:cont_Wass_ell}
		Due to Lemma~\ref{lm:H_diff_E} and the Kantorovich-Rubinstein duality of Wasserstein distance, 
		\begin{align}
			H_\ell(P) - H_\ell(Q) &\le \mathcal W_{|\cdot|}(P_{\ell(Z,a_Q)} , Q_{\ell(Z,a_Q)}) ,
		\end{align}
		and
		\begin{align}
			H_\ell(Q) - H_\ell(P) &\le \mathcal W_{|\cdot|}(P_{\ell(Z,a_P)} , Q_{\ell(Z,a_P)}) .
		\end{align}
	\end{theorem}

	\subsection{Bounds via $(\sA,\ell)$-dependent distance}\label{sec:H_diff_dlA}
	The bounds on entropy difference that have been studied so far are in terms of various statistical distances between $P$ and $Q$ or between $P_\ell$ and $Q_\ell$ that do not directly depend on the action space $\sA$.
	To obtain potentially tighter bounds, we consider distances that explicitly rely on both $\sA$ and $\ell$. 
	One such distance can be defined as follows.
	\begin{definition}
		The $(\sA,\ell)$-semidistance between distributions $P$ and $Q$ on $\sZ$ is defined as
		\begin{align}\label{eq:lA_dist}
			d_{\sA,\ell}(P,Q) \deq \sup_{a\in\sA} \big| \E_P[\ell(Z,a)] - \E_Q[\ell(Z,a)] \big | .
		\end{align}
	\end{definition}
	\noindent It can be checked that $d_{\sA,\ell}$ is symmetric and satisfies the triangle inequality, but it may happen that $d_{\sA,\ell}(P,Q)=0$ for $P\neq Q$, e.g. when $\ell\equiv 0$. For this reason, we call $d_{\sA,\ell}$ a semidistance.
	Note that $(\sA,\ell)$ also induces a class of functions 
	\begin{align}
		{\mathcal L}_{\sA,\ell} \deq \{\ell(\cdot,a):\sZ\rightarrow\R, a\in\sA\}, 
	\end{align}
	such that $d_{\sA,\ell}(P,Q)$ can be rewritten in terms of ${\mathcal L}_{\sA,\ell}$ as
	\begin{align}\label{eq:lA_dist_E}
		d_{\sA,\ell}(P,Q) = \sup_{f\in{\mathcal L}_{\sA,\ell}} \big| \E_P f - \E_Q f \big | .
	\end{align}
	We then see that $d_{\rm TV}(P,Q)$ is a special instance of $d_{\sA,\ell}(P,Q)$ with ${\mathcal L}_{\sA,\ell}$ being the set of measurable functions $f:\sZ\rightarrow[0,1]$. Additionally, $W_{\text{\tiny $\|\!\cdot\!\|$}}(P,Q)$ for $P$ and $Q$ on $\R^p$ with finite $\E_P\|Z\|$ and $\E_Q\|Z\|$ is another instance of $d_{\sA,\ell}(P,Q)$, with ${\mathcal L}_{\sA,\ell}$ being the set of 1-Lipschitz functions $f:\R^p\rightarrow \R$ with respect to the Euclidean distance.
	With the definition of $d_{\sA,\ell}(P,Q)$ in \eqref{eq:lA_dist} and Lemma~\ref{lm:H_diff_E}, it is straightforward to show the following bound on entropy difference.
	\begin{theorem}\label{prop:cont_dla}
		For distributions $P$ and $Q$ on $\sZ$,
		\begin{align}
			|H_{\sA,\ell}(P) - H_{\sA,\ell}(Q)| &\le d_{\sA,\ell}(P,Q) .
		\end{align}
	\end{theorem}
	\noindent We will find applications of this result in Section~\ref{sec:excess_risk_freq_lA}, where we study the excess risk of the ERM algorithm in frequentist statistical learning.
	
	
	\subsection{Bounds via Bregman divergence and Euclidean distance}\label{sec:H_diff_Breg}
	The bounds on entropy difference obtained in Sections~\ref{sec:H_diff_TV} to \ref{sec:H_diff_dlA} are all based on Lemma~\ref{lm:H_diff_E}, which is a relaxation of the variational representation of the generalized entropy. In this subsection, we take a different route to bound the entropy difference, by making use of the concavity of the generalized entropy.
	The concavity of $H_{\sA,\ell}(P)$ in $P$ can be seen from the definition in \eqref{eq:H_def}, as it is the infimum of a collection of linear functions of $P$. A Bregman divergence between distributions $P$ and $Q$ on a finite $\sZ$ \cite{BREGMAN1967} can thus be defined in terms of the negative generalized entropy, as
	\begin{align}\label{eq:Breg_def}
		d_{H}(P,Q) \deq H_{\sA,\ell}(Q) - H_{\sA,\ell}(P) + \nabla H_{\sA,\ell}(Q)^\top (P-Q) .
	\end{align}
	This definition gives two exact representations of the entropy difference in terms of Bregman divergence:
	\begin{align}
		H_{\sA,\ell}(P) - H_{\sA,\ell}(Q) &= \nabla H_{\sA,\ell}(Q)^\top (P-Q) - d_{H}(P,Q) \label{eq:H_diff_Breg1} \\
		&= \nabla H_{\sA,\ell}(P)^\top (P-Q) + d_H(Q,P) \label{eq:H_diff_Breg2}
	\end{align}
	where \eqref{eq:H_diff_Breg2} is obtained by exchanging the roles of $P$ and $Q$ in \eqref{eq:Breg_def}.
	With the Cauchy-Schwarz inequality, this leads to entropy difference bounds in terms of the Bregman divergence and the Euclidean distance between two distributions.
	\begin{theorem}\label{th:H_diff_Breg}
	For distributions $P$ and $Q$ on a finite $\sZ$, 
	\begin{align}
		H_{\sA,\ell}(P) - H_{\sA,\ell}(Q) \le 
		d_H(Q,P) + \| \nabla H_{\sA,\ell}(P) \| \|P-Q\| \label{eq:H_diff_ub_Breg} ,
	\end{align}
	where $d_H(Q,P)$ follows the definition in \eqref{eq:Breg_def}. Moreover,
	\begin{align}
	H_{\sA,\ell}(P) - H_{\sA,\ell}(Q) \le 
	\| \nabla H_{\sA,\ell}(Q) \| \|P-Q\| \label{eq:H_diff_ub_Euclidean} .
	\end{align}
	\end{theorem}
	\noindent{\bf Remark:}
	The upper bound in \eqref{eq:H_diff_ub_Euclidean} follows from \eqref{eq:H_diff_Breg1} and the nonnegativity of Bregman divergence, or it can be seen as a direct consequence of the concavity of the generalized entropy. By exchanging the roles of $P$ and $Q$, Theorem~\ref{th:H_diff_Breg} can also provide lower bounds for $H_{\sA,\ell}(P) - H_{\sA,\ell}(Q)$.
	
	As an example, we can use Theorem~\ref{th:H_diff_Breg} to bound the Shannon entropy difference. In this case, the Bregman divergence defined in \eqref{eq:Breg_def} coincides with the KL divergence $D(P\|Q)$. We have the following bounds.
	\begin{corollary}\label{co:H_log_diff_Breg}
		For distributions $P$ and $Q$ on a finite $\sZ$, 
		\begin{align}
			H_{\rm log}(P) - H_{\rm log}(Q) \le 
			D(Q \| P) + \| (-1 -\log P(z))_{z\in\sZ}\| \|P-Q\| \label{eq:H_log_diff_ub_Breg} .
		\end{align}
		Moreover,
		\begin{align}
			H_{\log}(P) - H_{\log}(Q) \le 
			\| (-1 -\log Q(z))_{z\in\sZ}\|  \|P-Q\| \label{eq:H_log_diff_ub_Euclidean} .
		\end{align}
	\end{corollary}
	\noindent Since Shannon entropy is permutation-invariant in the underlying distribution, $\|P-Q\|$ in \eqref{eq:H_log_diff_ub_Breg} and \eqref{eq:H_log_diff_ub_Euclidean} can be tightened by $\min_\Pi \|P-\Pi(Q)\|$, where $\Pi(Q)$ is a permutation of $Q$.

	\subsection{Comparison with existing bounds}\label{sec:H_diff_comp}
	To date there has been no general results for the continuity of generalized entropy.
	Existing entropy difference bounds in the literature are mainly for the Shannon entropy and the differential entropy.
	We make comparisons between the results presented in this work and some of the existing bounds.
	
	For Shannon entropy, the following well-known result provides an upper bound on the entropy difference in terms of total variation distance \!\cite[Lemma~2.7]{CsiKor_book1st}, \cite[Theorem~17.3.3]{Cover_book}.
	\begin{theorem}\label{lm:H_cont}
		For $P$ and $Q$ on a finite space $\sZ$ such that $d_{\rm TV}(P,Q)\le{1}/{4}$,
		\begin{align}\label{eq:ShH_cont_TV}
			|H_{\log}(P)-H_{\log}(Q)| \le 2 d_{\rm TV}(P,Q) \log \frac{|\sZ|}{2 d_{\rm TV}(P,Q)} .
		\end{align}
	\end{theorem}
	\noindent Compared with the upper bound \eqref{eq:log_abs_TV} in Corollary~\ref{co:log_TV} and the upper bounds \eqref{eq:H_ub_logChi2} and \eqref{eq:H_lb_logChi2} in Corollary~\ref{co:log_chi2}, we see that an advantage of the new upper bounds is that they do not require the distance between $P$ and $Q$ to be small to hold.
	While $\eqref{eq:log_abs_TV}$ requires the entries of the distributions to be bounded away from zero for the upper bound to be finite, \eqref{eq:H_ub_logChi2} and \eqref{eq:H_lb_logChi2} only require the varentropy of $Q$ and the cross varentropy of $P$ under $Q$ to be finite.
	Moreover, the upper bound in Corollary~\ref{co:log_TV} is tighter in $d_{\rm TV}(P,Q)$ when it is small. For example, if $d_{\rm TV}(Q_n,Q)$ is $ O(\frac{1}{n})$, then the upper bound in \eqref{eq:ShH_cont_TV} scales as $O(\frac{\log n}{n})$, while the upper bound in Corollary~\ref{co:log_TV} scales as $O(\frac{1}{n})$.
	
	Proved via an optimal coupling argument, another Shannon entropy difference bound appears in \cite{Zhang_entropy07} and states the following.
	\begin{theorem}\label{th:Zhang07}
		For distributions $P$ and $Q$ on a finite $\sZ$,
		\begin{align}\label{eq:Zhang07}
			|H_{\log}(P)-H_{\log}(Q)| \le d_{\rm TV}(P,Q) \log(|\sZ| -1) + h_2(d_{\rm TV}(P,Q))
		\end{align}
		where $h_2$ is the binary entropy function.
	\end{theorem}
	\noindent This bound has been generalized and improved in \cite{Ho_Yeung_entropy10} and \cite{Sason_entropy13}. While tighter than the bound in Theorem~\ref{lm:H_cont}, it still scales as $O(-d_{\rm TV}(P,Q)\log{d_{\rm TV}(P,Q)})$ when $d_{\rm TV}(P,Q)$ is small, hence not as tight as the bound in Corollary~\ref{co:log_TV} when $d_{\rm TV}(P,Q)$ approaches zero.
	As an example, for two Bernoulli distributions with biases $p$ and $q$, the white region in Fig.~\ref{fig:Bern_compare} indicates the collection of $(p,q)$ such that the bound in Corollary~\ref{co:log_TV} is tighter than the bound in Theorem~\ref{th:Zhang07}.
	\begin{figure}[h]
		\centering
		\includegraphics[scale = 0.6]{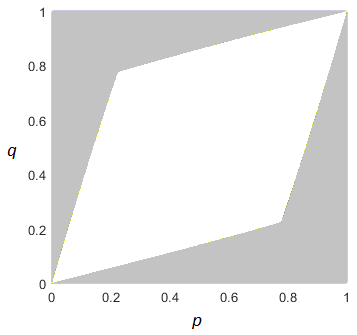}
		\caption{Comparison of bounds in \eqref{eq:log_abs_TV} and \eqref{eq:Zhang07} for ${\rm Bernoulli}(p)$ and ${\rm Bernoulli}(q)$: the bound in \eqref{eq:log_abs_TV} is tighter in the white region of $(p,q)$.}
		\label{fig:Bern_compare}
	\end{figure}
	
	For differential entropy, the entropy difference can be upper-bounded in terms of the Wasserstein distance, as stated in the following result \cite{PoliWu06}.
	\begin{theorem}\label{lm:h_cont}
		Let $\sZ=\R^p$. If $Q$ has a $(c_1,c_2)$-regular density, meaning that
		\begin{align}
			\| \nabla \log Q(z)\| \le c_1 \|z\| + c_2 ,\quad \forall z\in\R^p
		\end{align}
		then
		\begin{align}
			h(P) - h(Q) \le \Big(\frac{c_1}{2}\sqrt{\E_P[\|Z\|^2]} + \frac{c_1}{2}\sqrt{\E_Q[\|Z\|^2]} + c_2 \Big) W_{\|\cdot\|,2}(P,Q) ,
		\end{align}
		where $W_{\|\cdot\|,2}(P,Q)$ is the Wasserstein distance with respect to the Euclidean distance of order $2$.
	\end{theorem}
	\noindent Compared with the bound in \eqref{eq:H_ub_logChi2}, we see that \eqref{eq:H_ub_logChi2} only requires the varentropy of $Q$ to be finite, without other regularity conditions on $Q$.
	Moreover, the upper bound in \eqref{eq:H_ub_logChi2} depends on $P$ only through $\chi^2(P,Q)$, meaning that for a fixed $Q$, the upper bound is monotonically decreasing as $P$ gets closer to $Q$, which is sufficient to prove the upper semicontinuity of the entropy.
	
	For the quadratic loss, the following result given by Wu \cite{Wu_Var} upper-bounds the variance difference in terms of the Wasserstein distance. It can be proved by writing $\E_P[Z^2]$ and $\E_Q[Z^2]$ as $W_{\|\cdot\|,2}^2(P,\delta_0)$ and $W_{\|\cdot\|,2}^2(Q,\delta_0)$, and using the triangle inequality satisfied by the Wasserstein distance.
	\begin{theorem}\label{lm:Var_Wass}
		For $P$ and $Q$ on $\R$ with finite $\E_P[Z^2]$ and $\E_Q[Z^2]$,
		\begin{align}\label{eq:ub_Var_W2}
			\Var_P[Z]-\Var_Q[Z] \le 2 \Big(\sqrt{\E_P[Z^2]} + \sqrt{\E_Q[Z^2]}\Big) W_{\|\cdot\|,2}(P,Q) .
		\end{align}
	\end{theorem}
	\noindent Compared with \eqref{eq:Var_ub_Chi2}, the above upper bound only requires $P$ and $Q$ to have finite second moments, while \eqref{eq:Var_ub_Chi2} requires $Q$ to have a finite fourth moment. On the other hand, the upper bound in \eqref{eq:Var_ub_Chi2} depends on $P$ only through $\chi^2(P,Q)$, hence monotonically decreasing as $P$ gets closer to $Q$, which is sufficient to prove the upper semicontinuity.

	\subsection{An information-theoretic application: mutual information upper bound}\label{sec:App_mi_ub}
	As an application of the entropy difference bounds derived in the previous subsections, we prove new upper bounds for mutual information by applying Corollary~\ref{co:log_TV} and Corollary~\ref{co:cont_KL_bdd} 
	to the log loss.
	\begin{corollary}
		For jointly distributed random variables $X$ and $Z$ that can be either discrete or continuous, let
		\begin{align}
			\gamma(x) = \log\frac{\sup_{z\in\sZ}P_{Z|X=x}(z)}{\inf_{ z\in\sZ}P_{Z|X=x}(z)} 
		\end{align}
		be the range of variation of $\log P_{Z|X=x}(\cdot)$.
		Then from Corollary~\ref{co:cont_KL_bdd}, we have
		\begin{align}
			I(X;Z) \le  \sqrt{\frac{1}{2} \E\big[\gamma^2(X)\big]L(X;Z)} \bigwedge \frac{1}{2} \E\big[\gamma^2(X)\big] 
		\end{align}
		where $L(X;Z) = D(P_X P_Z \| P_{X,Z})$ is the Lautum information between $X$ and $Z$ \cite{Lautum}.
		Moreover, from Corollary~\ref{co:log_TV}, we have
		\begin{align}\label{eq:mi_ub_TV}
			I(X;Z) 
			\le  \Big(\sup\nolimits_{x\in\sX}\gamma(x)\Big) \int_{\sX}d_{\rm TV}(P_{Z|X=x} , P_Z)P_X(\rd x)  ,
		\end{align}
		where $\int_{\sX}d_{\rm TV}(P_{Z|X=x} , P_Z)P_X(\rd x)$ may be regarded as a total variation information.
	\end{corollary}
	\begin{proof}
		From the definition of mutual information,
		\begin{align}
			I(X;Z) &= H_{\log}(Z) - H_{\log}(Z|X) \\
			&= \int_{\sX} P_X(\rd x) (H_{\log}(P_Z) - H_{\log}(P_{Z|X=x})) . \label{eq:mi_H_diff}
		\end{align}
		If for any $x$, $\min_{ z\in\sZ} P_{Z|X=x}(z)>0$,
		then by Corollary~\ref{co:cont_KL_bdd},
		\begin{align}
			H_{\log}(P_Z) - H_{\log}(P_{Z|X=x}) &\le \gamma(x)\sqrt{\frac{1}{2} \big( D(P_{Z}\| P_{Z|X=x}) \wedge D(P_{Z|X=x}\| P_Z) \big) } .
		\end{align}
		Taking expectations on both sides over $X$, and using Cauchy-Schwarz inequality, we get
		\begin{align}
			I(X;Z) \le \sqrt{\frac{1}{2}\E\big[\gamma^2(X)\big] L(X;Z)} ,
		\end{align}
		and
		\begin{align}
			I(X;Z) \le \sqrt{\frac{1}{2}\E\big[\gamma^2(X)\big] I(X;Z)} .
		\end{align}
		The last inequality implies that
		\begin{align}
			I(X;Z) \le \frac{1}{2}\E\big[\gamma^2(X)\big] .
		\end{align}
		Finally, \eqref{eq:mi_ub_TV} follows from \eqref{eq:mi_H_diff} and Corollary~\ref{co:log_TV}.
	\end{proof}

	\section{Application to frequentist learning}\label{sec:App}
	Having studied the continuity property of the generalized entropy as a functional of the underlying distribution, we now apply the results obtained in Section~\ref{sec:H_diff} to the excess risk analysis of learning methods, the central problem of statistical learning theory.
	\subsection{Excess risk of ERM algorithm}\label{sec:excess_risk_freq}
	In the frequentist formulation of the statistical learning problem, there is a sample space $\sZ$, a fixed but unknown distribution $P$ on $\sZ$, and a hypothesis space $\sA$. A loss function $\ell:\sZ\times\sA\rightarrow\R$ is chosen to evaluate the hypotheses in $\sA$.
	For any hypothesis $a\in\sA$, its population risk is $\E_P[\ell(Z,a)]$.
	$H_{\sA,\ell}(P)$ is the \emph{minimum population risk} that would be achieved among $a\in\sA$ if $P$ were known. Neither $\E_P[\ell(Z,a)]$ nor $H_{\sA,\ell}(P)$ is known however, due to the lack of knowledge of $P$. 
	What is available instead is a training dataset $Z^n \deq (Z_1,\ldots,Z_n)$ of size $n$ drawn i.i.d.\ from $P$, with empirical distribution $\wh P_n$. 
	As a natural choice, the empirical risk minimization (ERM) algorithm returns a hypothesis $a_{\wh{P}_n}$ that minimizes the empirical risk $\E_{\wh P_n}[\ell(Z,a)]$ among $a\in\sA$, and the \emph{minimum empirical risk} is equal to $H_{\sA,\ell}(\wh P_n)$. Since $\wh P_n$ depends on $Z^n$, $H_{\sA,\ell}(\wh P_n)$ is a random variable.
	The entropy difference $|H_{\sA,\ell}(\wh P_n) - H_{\sA,\ell}(P)|$ tells us how well the unknown minimum population risk can be approximated by the minimum empirical risk that is known in principle.
	The results in Section~\ref{sec:H_diff} enable us to upper-bound $|H_{\sA,\ell}(\wh P_n) - H_{\sA,\ell}(P)|$ so as to evaluate the quality of this approximation.
	
	More importantly, the upper-bounding techniques developed in Sections~\ref{sec:H_diff_TV} to \ref{sec:H_diff_dlA} provide us with a means to analyze the \emph{excess risk} of the ERM algorithm, defined as the gap between the population risk of the algorithm-returned hypothesis $a_{\wh P_n}$ and the minimum population risk:
	\begin{align}\label{eq:excess_risk_freq}
		R_{\rm excess} \deq \E_P\big[\ell(Z,a_{\wh{P}_n}) | Z^n\big] - H_{\sA,\ell}(P) ,
	\end{align}
	where $Z$ is a fresh sample from $P$ independent of $Z^n$, so that $P_{Z|Z^n}= P$.
	Note that $R_{\rm excess}$ is a random variable, since $\E_P\big[\ell(Z,a_{\wh{P}_n}) | Z^n\big]$ depends on $Z^n$ through $a_{\wh{P}_n}$.
	Writing $R_{\rm excess}$ as
	\begin{align}
		R_{\rm excess}
		&= \big(\E_P\big[\ell(Z,a_{\wh{P}_n}) | Z^n\big] - H_{\sA,\ell}(\wh P_n) \big) + \big( H_{\sA,\ell}(\wh P_n) - H_{\sA,\ell}(P) \big) ,
	\end{align}
	and using the fact that all the entropy difference bounds in Sections~\ref{sec:H_diff_TV} to \ref{sec:H_diff_dlA} are based on Lemma~\ref{lm:H_diff_E}, and the fact that every upper bound for $H_{\sA,\ell}(P) - H_{\sA,\ell}(\wh{P}_n)$ obtained based on Lemma~\ref{lm:H_diff_E} also upper-bounds $\E_P\big[\ell(Z,a_{\wh{P}_n}) | Z^n\big] - H_{\sA,\ell}(\wh P_n)$, we deduce the following result.
	\begin{lemma}\label{lm:ERM_Rexcess}
		For any almost-sure upper bound $B$ for $|H_{\sA,\ell}(\wh P_n) - H_{\sA,\ell}(P)|$ obtained based on Lemma~\ref{lm:H_diff_E}, in particular based on the results in Sections~\ref{sec:H_diff_TV} to \ref{sec:H_diff_dlA}, almost surely we have
		\begin{align}
			R_{\rm excess} \le 2B .  \label{eq:excess_risk_freq_ub}  
		\end{align}
	\end{lemma}
	\noindent 
	
	We give three examples for the application of Lemma~\ref{lm:ERM_Rexcess}, using different upper bounds for the entropy difference derived in Section~\ref{sec:H_diff}.

	\subsection{Finite sample space}
	When the sample space $\sZ$ has a finite number of elements, we can make use of the entropy difference upper bounds in terms of total variation distance (Theorem~\ref{prop:cont_TV}) and KL divergence (Corollary~\ref{co:cont_KL_bdd}). 
	The resulting upper bounds for the excess risk hold virtually for \emph{any} hypothesis space $\sA$.
	For simplicity, we consider the case where the loss function takes values in $[0,1]$.
	\begin{theorem}\label{th:ERM_TV}
		If $\sZ$ is finite and $\ell(z,a)\in [0,1]$ for all $(z,a)\in\sZ\times\sA$, then for any $\sA$,
		\begin{align}\label{eq:Rexcess_ub_TV}
			\E[R_{\rm excess}] \le \sqrt{\frac{|\sZ|}{n}} ;
		\end{align}
		and for any $\eps > 0$,
		\begin{align}\label{eq:Rexcess_ub_KL}
			\PP[R_{\rm excess} > \eps] \le \exp\Big\{-n\Big(\frac{\eps^2}{2} - \frac{|\sZ|\log(n+1)}{n} \Big) \Big\} .
		\end{align}
	\end{theorem}
	\begin{proof}
		The upper bound in \eqref{eq:Rexcess_ub_TV} is a consequence of Lemma~\ref{lm:ERM_Rexcess}, Theorem~\ref{prop:cont_TV}, and the fact that $\E[2 d_{\rm TV}(\wh P_n,P)]\le \sqrt{{|\sZ|}/{n}}$ \cite[Lemma~5]{Berend_TV13}.
		The upper bound in \eqref{eq:Rexcess_ub_KL} is a consequence of Lemma~\ref{lm:ERM_Rexcess}, Corollary~\ref{co:cont_KL_bdd}, and the fact that $\PP[D(\wh P_n \| P) > \eps]\le \exp\{-n(\eps - \frac{|\sZ|\log(n+1)}{n} ) \}$ \cite[Theorem~11.2.1]{Cover_book}.
	\end{proof}
	\noindent{\bf Remark.}
	The upper bounds in Theorem~\ref{th:ERM_TV} can be extended to the case where $\sZ$ is countably infinite, using the results in \cite[Lemma~8 and Theorem~3]{conv_emp_TV12}. 
	In addition, via Pinsker's inequality, the upper bound in \eqref{eq:Rexcess_ub_KL} can be used to bound $\PP[d_{\rm TV}(\wh P_n , P) > \eps]$, which complements the results in \cite[Theorem~3]{conv_emp_TV12} and \cite[Lemma~3]{Devroye1983} on the convergence of empirical distribution in the total variation distance.
	
	To evaluate the upper bounds in Theorem~\ref{th:ERM_TV}, consider the problem of binary classification, where $\sZ = \sX \times \sY$ with $\sY=\{0,1\}$.
	Let $\sA$ be the space of \emph{all mappings} from $\sX$ to $\sY$, and $\ell(z,a) = \I\{y \neq a(x)\}$.
	From \eqref{eq:Rexcess_ub_TV}, we get an upper bound for the expected excess risk of the ERM algorithm,
	\begin{align}\label{eq:excess_binary}
		\E[R_{\rm excess}] \le \sqrt{\frac{2|\sX|}{n}} .
	\end{align}
	This bound is even better in prefactor than the bound
	$
	\E[R_{\rm excess}] \le 8\sqrt{\frac{|\sX|\log 2}{n}} 
	$
	given by the popular Rademacher complexity analysis, 
	which is a consequence of the fact that the cardinality of the hypothesis class $\sA$ is $2^{|\sX|}$ when $\sX$ is finite \cite{ShBe_book14}.

	\subsection{Lipschitz-continuous loss function}
	When the loss function is Lipschitz-continuous in $z$ for all $a$, where $z$ can be continuous-valued, we can use the bound in Theorem~\ref{prop:cont_Wass_dual} in terms of the Wasserstein distance to bound the excess risk.
	\begin{theorem}\label{prop:excess_Lip}
		Let $\sZ = \sX\times\sY$ where $\sY=[-b,b]$ and $\sX \subset \R^p$ with $p > 1$.
		Suppose that $\E[\|X\|^2]$ is finite under the unknown distribution.
		Consider an action space $\sA\subset\R^k$ with an arbitrary $k$, and a function $f:\sX\times\sA\rightarrow [-b,b]$ such that $f(\cdot,a)$ is $\rho_f$-Lipschitz in $x$ with respect to the Euclidean distance for all $a\in\sA$. 
		Then for the loss function $\ell_1(z,a) = |y - f(x,a)|$,
		\begin{align}
			\E[R_{\rm excess}] \le c (\rho_f \vee 1) {\E\|Z\|}  n^{-1/(p+1)} ;
		\end{align}
		while for the loss function $\ell_2(z,a) = (y - f(x,a))^2$,
		\begin{align}
			\E[R_{\rm excess}] \le 4 c b (\rho_f \vee 1) {\E\|Z\|} n^{-1/(p+1)} ,
		\end{align}
		where $c$ is an absolute constant. 
	\end{theorem}
	\begin{proof}
		We first show that the Lipschitz continuity of $f(\cdot,a)$ in $x$ can be translated to the Lipschitz continuity of $|y-f(x,a)|$ in $z=(x,y)$.
		For any $a\in\sA$, and any $z,z'\in\sZ$,
		\begin{align}
			\big| |y-f(x,a)| - |y'-f(x',a)| \big| 
			&\le \big|y-f(x,a) - y' + f(x',a)\big| \\
			&\le |f(x,a) - f(x',a)| + |y-y'| \label{eq:LipWass1} \\
			&\le \rho_f\|x-x'\| + |y-y'| \\
			&\le \sqrt{2} (\rho_f \vee 1)\|z-z'\| , \label{eq:LipWass3}
		\end{align}
		where in \eqref{eq:LipWass3} we used the fact that $u+v\le \sqrt{2u^2 + 2v^2}$ for $u,v\in\R$.
		It implies that $\ell_1(z,a)=|y-f(x,a)|$ is $\sqrt{2} (\rho_f \vee 1)$-Lipschitz in $z=(x,y)$ for all $a\in\sA$.
		Since $|y-f(x,a)|\in[0,2b]$, it further implies that $\ell_2(z,a)=(y-f(x,a))^2$ is $4\sqrt{2}b (\rho_f \vee 1)$-Lipschitz in $z$ for all $a\in\sA$.
		It follows from Lemma~\ref{lm:ERM_Rexcess} and Theorem~\ref{prop:cont_Wass_dual} that for $\ell_1(z,a)=|y-f(x,a)|$,
		\begin{align}
			R_{\rm excess} \le 2\sqrt{2}(\rho_f \vee 1) W_{\text{\tiny $\|\!\cdot\!\|$}}(\wh P_n, P) ;
		\end{align}
		while for $\ell_2(z,a)=(y-f(x,a))^2$,
		\begin{align}
			R_{\rm excess} \le 8\sqrt{2}b(\rho_f \vee 1) W_{\text{\tiny $\|\!\cdot\!\|$}}(\wh P_n, P) .
		\end{align}
		The proof is completed with a result on the Wasserstein convergence of the empirical distribution \cite[Theorem~3.1]{lei2020}\cite[Proposition~10]{weed2019}, which states that for a distribution $P$ on $\sZ\subset\R^{p+1}$ with $p>1$,
		\begin{align}
			\E[ W_{\text{\tiny $\|\!\cdot\!\|$}}(\wh P_n, P)] \le c' \E[\|Z\|] n^{-1/{(p+1)}} ,
		\end{align}
		where $c'$ is some absolute constant. 
	\end{proof}
	We see that the upper bound in Theorem~\ref{prop:excess_Lip} does not depend on the dimension of $\sA$, and converges to zero as $n\rightarrow\infty$ for any fixed dimension $p$ of $\sX$; however, the rate of convergence suffers from the curse of dimensionality in $p$. An open question is whether there is a way to leverage the results in Section~\ref{sec:H_diff_P_ell} to bounding the excess risk in terms of statistical distances between the distributions of $\ell(Z,a_{\wh p_n})$ when $Z$ is drawn from $P$ and from $\wh P_n$.
	It may lead to tighter bounds when $f$ in Theorem~\ref{prop:excess_Lip} has additional regularities beyond being Lipschitz in $x$.
	This question is partially addressed by looking into a statistical distance that compares the expected loss under distributions $P$ and $\wh P_n$, but at a worst hypothesis in $\sA$, as discussed in the next subsection.

	\subsection{Learnability, typicality, and entropy continuity}\label{sec:excess_risk_freq_lA}
	The results in the two preceding subsections can be unified by considering the entropy difference bound via the $(\sA,\ell)$-semidistance defined in \eqref{eq:lA_dist}. We have
	\begin{align}
		d_{\sA,\ell}(\wh P_n,P) = \sup_{a\in\sA} \big| \E_{\wh P_n}[\ell(Z,a)] - \E_P[\ell(Z,a)] \big | ,
	\end{align}
	which is essentially the \emph{uniform deviation} of the empirical risk from the population risk with respect to $(\sA,\ell)$.
	It follows from Lemma~\ref{lm:ERM_Rexcess} and Theorem~\ref{prop:cont_dla} that
	\begin{align}\label{eq:Rexcess_dAl}
		R_{\rm excess} &\le 2d_{\sA,\ell}(\wh P_n,P) \quad{\rm a.s.}
	\end{align}
	This result recovers the classic upper bound on the excess risk of the ERM algorithm in terms of the uniform deviation \cite{DGLbook96}.
	
	The conditions on the convergence of the uniform deviation to zero, 
	\begin{align}\label{eq:UCEM}
		d_{\sA,\ell}(\wh P_n,P) \xrightarrow{\text{a.s.}} 0 \quad \text{as $n\rightarrow\infty$}
	\end{align}
	have been well-studied in the mathematical statistics and statistical learning theory literature as a form of uniform law of large numbers \cite{DGLbook96,ShBe_book14}.
	Recall that $d_{\sA,\ell}$ can also be defined with respect to the function class ${\mathcal L}_{\sA,\ell} = \{\ell(\cdot,a), a\in\sA\}$ induced by $(\sA,\ell)$ as shown in \eqref{eq:lA_dist_E}, namely
	\begin{align}
		d_{\sA,\ell}(\wh P_n,P) = \sup_{f\in{\mathcal L}_{\sA,\ell}} \big| \E_{\wh P_n}[f(Z)] - \E_P[f(Z)] \big | .
	\end{align}
	The function class ${\mathcal L}_{\sA,\ell}$ is called a \emph{Glivenko-Cantelli (GC) class} if \eqref{eq:UCEM} holds for every distributon $P$ on $\sZ$, c.f. \cite{TypicalRaginsky}. 
	Further, the hypothesis space $\sA$ is said to be \emph{learnable} with respect to $\ell$ if ${\mathcal L}_{\sA,\ell}$ is a GC class.
	Theorem~\ref{th:ERM_TV} and Theorem~\ref{prop:excess_Lip} each involves a special instance of the GC class that has virtually no restriction on $\sA$: one with all measurable functions $\sZ\rightarrow [0,1]$ and a \emph{finite} $\sZ$, such that 
	$$
	d_{\sA,\ell}(\wh P_n,P)=d_{\rm TV}(\wh P_n,P)\xrightarrow{\text{a.s.}} 0 ;
	$$
	and the other with all bounded Lipschitz-continuous functions $\R^{p+1} \rightarrow [-b,b]$ with a common Lipschitz constant, such that 
	$$
	d_{\sA,\ell}(\wh P_n,P) \propto W_{\text{\tiny $\|\!\cdot\!\|$}}(\wh P_n, P)\xrightarrow{\text{a.s.}} 0  .
	$$
	In general, a GC class and the rate of convergence in \eqref{eq:UCEM} rely on the properties of both $\sA$ and $\ell$.
	A well-known example of such a GC class is the class of indicator functions of a special collection of subsets of $\sZ$ which has a finite Vapnik-Chervonenkis (VC) dimension \cite{DGLbook96}. For this class, with $\ell$ being the zero-one loss, and $\sA$ being the collection of subsets of $\sZ$ with a finite VC dimension $V(\sA)$, $\E[d_{\sA,\ell}(\wh P_n,P)]$ explicitly depends on $\sA$ through 
	\begin{align}
		\E[d_{\sA,\ell}(\wh P_n,P)] \sim O\big(\sqrt{{V(\sA)}/{n}}\big) .
	\end{align}
	
	Conceptually, given $\sA$ and $\ell$, we can also define the $(\sA,\ell)$-\emph{typical set} of elements in $\sZ^n$ according to $d_{\sA,\ell}(\wh P_n,P)$ as in \cite[Definition~4]{TypicalRaginsky},
	\begin{align}\label{eq:def_Aell_typical}
		{\mathcal T}_{\sA,\ell}(P,n,\eps) \deq \big\{z^n\in\sZ^n: d_{\sA,\ell}(\wh P_n,P) \le \eps\big\} , \quad \eps>0 .
	\end{align}
	In words, a dataset $z^n$ is $(\sA,\ell)$-typical if the empirical risks on it, uniformly for all hypotheses in $\sA$, are close to the corresponding population risks.
	As a consequence of Theorem~\ref{prop:cont_dla} in Section~\ref{sec:H_diff_dlA}, the minimum empirical risk on this typical set can closely approximate the minimum population risk, as $|H_{\sA,\ell}(\wh P_n) - H_{\sA,\ell}(P)|\le \eps$; moreover, from \eqref{eq:Rexcess_dAl}, the ERM algorithm with an input drawn from this typical set will output a near-optimal hypothesis, as $R_{\rm excess}\le 2\eps$.
	For example, when $\sA$ and $\ell$ are such that $\mathcal{L}_{\sA,\ell}$ is the set of measurable functions $\sZ\rightarrow [0,1]$, the $(\sA,\ell)$-typical set defined in \eqref{eq:def_Aell_typical} reduces to the one characterized by the total variation distance between $\wh P_n$ and $P$,
	\begin{align}
		{\mathcal T}_{\rm TV}(P,n,\eps) = \big\{z^n\in\sZ^n: d_{\rm TV}(\wh P_n,P) \le \eps\big\}
	\end{align}
	which is proposed and used in \cite{coordinate_capacity}. When $\sZ$ is finite, the above typical set is almost equivalent to the notion of \emph{strong typicality} commonly used in information theory \cite{CsiKor_book1st} \cite[(10.106)]{Cover_book} as shown in \cite{coordinate_capacity}, and will include almost all elements in $\sZ^n$ as $n\rightarrow\infty$.
	Theorem~\ref{th:ERM_TV} can thus be understood from the viewpoint of strong typitcality as well, in that eventually almost every sequence has an empirical distribution close to $P$.
	In general, the definition of ${\mathcal T}_{\sA,\ell}(P,n,\eps)$ applies to uncountably infinite $\sZ$ as well.
	We then have the following connection among typicality, entropy continuity, and learnability:
	if ${\mathcal L}_{\sA,\ell}$ is a GC class, then for any $\eps>0$, as $n\rightarrow \infty$,
	\begin{align}
		\PP\big[ {\mathcal T}_{\sA,\ell}(P,n,\eps)\big]\rightarrow 1 
	\end{align}
	by the definition in \eqref{eq:def_Aell_typical}, which implies that
	\begin{align}
		\PP\big[|H_{\sA,\ell}(\wh P_n)-H_{\sA,\ell}(P)|\le \eps\big]\rightarrow 1 
	\end{align}
	by Theorem~\ref{prop:cont_dla},
	which further implies that
	\begin{align}
		\PP\big[R_{\rm excess}\le 2\eps\big]\rightarrow 1
	\end{align}
	by Lemma~\ref{lm:ERM_Rexcess}. 
	The rate of convergence will depend on $\sA$ and $\ell$ in general.

	\section{Application to Bayesian learning}\label{sec:MER_Bayes}
	Another application of the results in Section~\ref{sec:H_diff} to statistical learning is the analysis of the minimum excess risk in Bayesian learning.
	This problem is formulated and studied in detail in \cite{MER19} using several different approaches. Here we give an overview of the analysis based on the entropy continuity presented in \cite[Section~4]{MER19}.
	\subsection{Minimum excess risk in Bayesian learning}
	As an alternative to the frequentist formulation of the learning problem, \emph{Bayesian learning} under a parametric generative model assumes that the data $Z^n=((X_1,Y_1),\ldots,(X_n,Y_n))$, with $Z_i\deq (X_i,Y_i)$, is generated from a member of a parametrized family of probabilistic models $\{P_{Z|w}, w\in\sW\}$, where the model parameter $W$ is an unknown random element in $\sW$ with a prior distribution $P_W$. With a fresh sample $Z=(X,Y)$, $X$ is observed, and the goal is to predict $Y$ based on $X$ and $Z^n$.
	Formally, the joint distribution of the model parameter, the dataset and the fresh sample is
	\begin{align}\label{eq:joint_dist}
		P_{W, Z^n, Z} = P_W  \Big(\prod\limits_{i=1}^n P_{Z_i|W}\Big)  P_{Z|W},
	\end{align}
	where $P_{Z_i|W} = P_{Z|W}$ for each $i$.
	Given an action space $\sA$ and a loss function $\ell : \sY \times \sA \rightarrow \R$, the goal of Bayesian learning can be phrased	as seeking a {\em decision rule} $\psi: \sX \times \sZ^n \rightarrow \sA$ to make the expected loss $\E[\ell(Y,\psi(X,Z^n))]$ small.
	In contrast to the frequentist learning, since the joint distribution $P_{Z^n,Z}$ is known, the search space here is all decision rules such that $\E[\ell(Y,\psi(X,Z^n))]$ is defined, i.e. all measurable functions $\sX \times \sZ^n \rightarrow \sA$, without being restricted to a hypothesis space.
	The minimum achievable expected loss is called the \emph{Bayes risk} in Bayesian learning:
	\begin{align}\label{eq:RB_def}
		H_\ell(Y|X,Z^n) = \inf_{\psi: \sX \times \sZ^n \rightarrow \sA} \E[\ell(Y, \psi(X,Z^n))] ,
	\end{align}
	which is essentially the generalized conditional entropy of $Y$ given $(X,Z^n)$ in view of the definition in \eqref{eq:cond_H_def}.
	As shown by a data processing inequality for the Bayes risk \cite[Lemma~1]{MER19}, $H_\ell(Y|X,Z^n)$ decreases as the data size $n$ increases.
	The fundamental limit of the Bayes risk can be defined as the minimum expected loss when the model parameter $W$ is known:
	\begin{align}\label{eq:RBW_def}
		H_\ell(Y|X,W) = \inf_{\Psi: \sX\times\sW \rightarrow \sA} \E[\ell(Y,\Psi(X,W))] .
	\end{align}
	The \emph{minimum excess risk} (MER) in Bayesian learning is defined as the gap between the Bayes risk and its fundamental limit, which is the minimum achievable excess risk among all decision rules:
	\begin{align}\label{eq:MER_def}
		{\rm MER}_\ell \deq H_\ell(Y|X,Z^n) - H_\ell(Y|X,W) .
	\end{align}
	\noindent The MER is an algorithm-independent quantity. Its value and rate of convergence quantify the difficulty of the \emph{learning} problem, which is due to the lack of knowledge of $W$. It can serve as a formal definition of the minimum \textit{epistemic uncertainty}, with $H_\ell(Y|X,W)$ serving as the definition of the \emph{aleatoric uncertainty}, which have been only empirically studied so far \cite{NIPS2017_7141,Hllermeier2019Aleatoric}.

	\subsection{Method of analysis based on entropy continuity}
	In what follows, we outline the idea of how the upper bounds on entropy difference derived in Section~\ref{sec:H_diff} can be used to upper-bound the MER. We consider the predictive modeling framework, a.k.a.\ probabilistic discriminative model, where $P_{Z|W} = P_{X|W}  K_{Y|X,W}$, with the probability transition kernel $K_{Y|X,W}$ directly describing the predictive model of the quantity of interest given the observation.
	First, we have the following lemma that bounds the deviation of the posterior predictive distribution $P_{Y|X,Z^n}$ from the true predictive model $K_{Y|X,W}$, which is a simple consequence of the convexity of the statistical distance under consideration.
	\begin{lemma}\label{lm:ub_f_div} Let $W'$ be a sample from the posterior distribution $P_{W|X,Z^n}$, such that $W$ and $W'$ are conditionally i.i.d.\ given $(X,Z^n)$.
		Then for any $f$-divergence or Wasserstein distance $D$,
		\begin{align}\label{eq:post_ub_f_div_E}
			\E[ D(P_{Y|X,Z^n} , K_{Y|X,W}) ] \le \E[ D(K_{Y|X,W'} , K_{Y|X,W}) ] 
		\end{align}
		where the expectations are taken over the conditioning variables according to the joint distribution of $(W,W',X,Z^n)$.
	\end{lemma}
	\noindent 
	
	The main utility of Lemma~\ref{lm:ub_f_div} is that, whenever $D(K_{Y|x,w'} , K_{Y|x,w})$ can be upper-bounded in terms of $\|w'-w\|^2$, we can invoke the fact that
	\begin{align}\label{eq:2MMSE}
		\E[\| W'-W \|^2] = 2 H_2(W|X,Z^n) 
	\end{align}
	as a consequence of the orthogonality principle in the MMSE estimation \cite{Liu17,Bhatt18,BH_RPbook}, so that the expected deviation $\E[ D(P_{Y|X,Z^n} , K_{Y|X,W}) ]$ can be bounded in terms of $H_2(W|X,Z^n)$, the MMSE of estimating $W$ from $(X,Z^n)$.
	Lemma~\ref{lm:ub_f_div} and \eqref{eq:2MMSE} give us a route to bounding the MER in terms of $H_2(W|X,Z^n)$, provided we can bound the entropy difference in \eqref{eq:MER_def} in terms of $D(P_{Y|X,Z^n} , K_{Y|X,W})$. The latter problem is precisely the subject of Section~\ref{sec:H_diff}.
	
	\subsection{Example: Bayesian logistic regression with zero-one loss}
	We give an example where the results in Section~\ref{sec:H_diff} can be applied to the analysis of Bayesian logistic regression with zero-one loss.
	Bayesian logistic regression is an instance under the predictive modeling framework, where $\sY = \{0,1\}$, $W\in\R^d$ is assumed to be independent of $X$, and the predictive model is specified by 
	$
	K_{Y|x,w}(1) = \sigma( w^\top \phi(x)) 
	$, 
	with $\sigma(a) \deq 1/(1+e^{-a})$, $a\in\R$, being the logistic sigmoid function, and $\phi(x) \in \R^d$ being the feature vector of the observation.
	
	For the zero-one loss, whenever $\sY$ is discrete, we have
	\begin{align}
		{\rm MER}_{01} &= \E[\max\nolimits_{y\in\sY}K_{Y|X,W}(y)] - \E[\max\nolimits_{y\in\sY}P_{Y|X,Z^n}(y)]  \\
		&= \int \big( \max\nolimits_{y\in\sY}K_{Y|x,w}(y) - \max\nolimits_{y\in\sY}P_{Y|x,z^n}(y) \big) P({\rm d}w, {\rm d}x, {\rm d}z^n) \\
		&\le  \int d_{\rm TV}(K_{Y|x,z^n}, P_{Y|x,w} ) P({\rm d}w, {\rm d}x, {\rm d}z^n) \label{eq:pf_ub_R01_1} \\
		&\le  \E[d_{\rm TV}( K_{Y|X,W'}, K_{Y|X,W})] \label{eq:pf_ub_R01_2}
	\end{align}
	where \eqref{eq:pf_ub_R01_1} follows from Theorem~\ref{prop:cont_TV}, and \eqref{eq:pf_ub_R01_2} follows from Lemma~\ref{lm:ub_f_div}.
	With the predictive model specified above, as $\| \nabla_w \sigma(w^\top \phi(x)) \| \le \|\phi(x)\|/4$, we know that $ \sigma(w^\top \phi(x)) $ is $ \|\phi(x)\|/4 $-Lipschitz in $w$, hence
	\begin{align}
		d_{\rm TV}(K_{Y|x,w'}, K_{Y|x,w}) 
		= \big| \sigma(w'^\top \phi(x)) -  \sigma(w^\top \phi(x)) \big| 
		\le \frac{1}{4} \|\phi(x)\| \|w' - w\| \label{eq:TV_logistic} . 
	\end{align}
Consequently, the MER with respect to zero-one loss satisfies
		\begin{align}
			{\rm MER}_{01} 
			&\le \E[d_{\rm TV}( K_{Y|X,W'}, K_{Y|X,W})] \\
			&\le \frac{1}{4} \E\big[\|\phi(X)\| \|W' - W\|\big] \\
			&\le \frac{1}{4}  \E[\|\phi(X)\|] \sqrt{\E\big[\|W'-W\|^2\big] } \\
			&= \frac{1}{4} \E[\|\phi(X)\|] \sqrt{ 2 H_2(W|Z^n) } \label{eq:MER01_logistic} 
		\end{align}
		where the last step is due to \eqref{eq:2MMSE} and the assumption that $W$ and $X$ are independent.
		
	This result explicitly shows that the MER in logistic regression depends on how well we can estimate the model parameters from data, as it is dominated by $H_2(W|Z^n)$, the MMSE of estimating $W$ from $Z^n$.
	A closed-form expression for this MMSE may not exist; nevertheless, any upper bound on it that is nonasymptotic in $n$ will translate to a nonasymptotic upper bound on the MER. 
	Moreover, this result explicitly shows how the model uncertainty due to the estimation error of the model parameters translates to the MER under the zero-one loss, which represents the minimum epistemic uncertainty, and how it then contributes to the minimum overall prediction uncertainty, which is the sum of the MER and the aleatoric uncertainty $\E[\min\{\sigma( W^\top \phi(X)), 1 - \sigma( W^\top \phi(X))\}]$.
	It thus provides a theoretical guidance on \emph{uncertainty quantification} in Bayesian learning, which is an increasingly important direction of research with wide range of applications.

	\section{Application to inference and learning with distribution shift}\label{sec:App_mismatch}
	Based on Lemma~\ref{lm:H_diff_E}, we have developed a number of approaches to bounding the difference of the generalized \emph{unconditional} entropy in Section~\ref{sec:H_diff}. We also studied the applications of the results in both frequentist learning and Bayesian learning in the two preceding sections. The idea behind Lemma~\ref{lm:H_diff_E} can be extended to bounding the difference of the generalized \emph{conditional} entropy defined in \eqref{eq:cond_H_def}.
	In this section, we work out this extension to derive performance bounds for Bayes decision making under a mismatched distribution. The results can be applied to analyzing the excess risk in learning by first projecting the empirical distribution to a predefined family of distributions and then using the projection as a surrogate of the data-generating distribution for decision making.

	\subsection{Bounds on conditional entropy difference}
	Consider the Bayes decision-making problem under which the generalized conditional entropy is defined as in \eqref{eq:cond_H_def}.
	Let $P=P_X P_{Y|X}$ and $Q=Q_X Q_{Y|X}$ be two joint distributions on $\sX\times\sY$.
	Given an action space $\sA$ and a loss function $\ell:\sY\times\sA\rightarrow\R$, let $\psi_P:\sX\rightarrow\sA$ and $\psi_Q:\sX\rightarrow\sA$ be the Bayes decision rules with respect to $(\sA,\ell)$ under $P$ and $Q$ respectively, such that
	$
	H_{\ell}(P_{Y|X} | P_X ) = \E_{P}[\ell(Y,\psi_P(X))] 
	$
	and
	$
	H_{\ell}(Q_{Y|X} | Q_X ) = \E_{Q}[\ell(Y,\psi_Q(X))] . 
	$
	Note that $\psi_P(x)$ and $\psi_Q(x)$ are the optimal actions that achieve the generalized unconditional entropy of $P_{Y|X=x}$ and $Q_{Y|X=x}$ respectively.
	Then, in the same spirit of Lemma~\ref{lm:H_diff_E}, we have the following result for the difference between generalized conditional entropy.
	\begin{lemma}\label{lm:cond_H_diff_E}
		Let $P=P_X P_{Y|X}$ and $Q=Q_X Q_{Y|X}$ be two joint distributions on $\sX\times\sY$. Then the difference between the generalized conditional entropy with respect to $(\sA,\ell)$ under $P$ and $Q$ satisfy
		\begin{align}
			H_{\ell}(P_{Y|X} | P_X ) - H_{\ell}(Q_{Y|X} | Q_X )
			&\le \E_P[\ell(Y,\psi_Q(X))] - \E_Q [\ell(Y,\psi_Q(X))] \label{eq:cond_H_diff_E+}
		\end{align}
		and
		\begin{align}
			H_{\ell}(Q_{Y|X} | Q_X ) - H_{\ell}(P_{Y|X} | P_X )
			&\le \E_Q[\ell(Y,\psi_P(X))] - \E_P [\ell(Y,\psi_P(X))] . \label{eq:cond_H_diff_E-}
		\end{align}
	\end{lemma}
	With Lemma~\ref{lm:cond_H_diff_E}, all the results developed in Sections~\ref{sec:H_diff_TV} to \ref{sec:H_diff_dlA} on the entropy difference can be extended to bounds for the conditional entropy difference.
	For example, the results in Sections~\ref{sec:H_diff_TV}, \ref{sec:H_diff_KL} and \ref{sec:H_diff_Chi2} can be extended by replacing $a_Q$ and $a_P$ by $\psi_Q(X)$ and $\psi_P(X)$ respectively, in both the conditions and the bounds, and by replacing the statistical distances between $P$ and $Q$ by distances between $P_{X,Y}$ and $Q_{X,Y}$.
	In view of Theorem~\ref{th:D_PQ_ell} in Section~\ref{sec:H_diff_P_ell}, the statistical distances between $P$ and $Q$ can even be replaced by distances between $P_{\ell(Y,\psi_Q(X))}$ and $Q_{\ell(Y,\psi_Q(X))}$, or between $P_{\ell(Y,\psi_P(X))}$ and $Q_{\ell(Y,\psi_P(X))}$.
	In view of the results in Section~\ref{sec:H_diff_Wass}, we can also bound the conditional entropy difference by the Wasserstein distance between $P_{X,Y}$ and $Q_{X,Y}$ if $\ell(Y,\psi_Q(X))$ or $\ell(Y,\psi_P(X))$ is Lipschitz in $(X,Y)$.
	Moreover, we can define an $(\sA,\ell)$-semidistance between $P_{X,Y}$ and $Q_{X,Y}$ as 
	\begin{align}
		d_{\sA,\ell}(P_{X,Y},Q_{X,Y}) \deq \sup_{\psi:\sX\,\rightarrow\sA} \big| \E_P[\ell(Y,\psi(X))] - \E_Q[\ell(Y,\psi(X))] \big| ,
	\end{align}
	and use it to bound the conditional entropy difference, similar to the results in Section~\ref{sec:H_diff_dlA}.
	
	As an illustrative example, suppose the loss function $\ell(y,a)\in [0,1]$ for all $(y,a)\in\sY\times\sA$. Then
	\begin{align}
		H_{\ell}(P_{Y|X} | P_X ) - H_{\ell}(Q_{Y|X} | Q_X )
		&\le \E_P[\ell(Y,\psi_Q(X))] - H_{\ell}(Q_{Y|X} | Q_X )  \label{eq:cond_H_diff_1} \\
		&\le \sqrt{\frac{1}{2}  D(P_{X,Y} \| Q_{X,Y} )}  \label{eq:cond_H_diff_2} \\
		&= \sqrt{\frac{1}{2}  \big(D(P_X \| Q_X)+D(P_{Y|X} \| Q_{Y|X} | P_X) \big)} \label{eq:cond_H_diff_3},
	\end{align}
	where \eqref{eq:cond_H_diff_1} is due to Lemma~\ref{lm:cond_H_diff_E}; \eqref{eq:cond_H_diff_2} is due to \eqref{eq:E_diff_KL_subG}; and \eqref{eq:cond_H_diff_3} follows from the chain rule of KL divergence.
	Not only serving as an upper bound for the conditional entropy difference, the result also implies that when both $D(P_X \| Q_X)$ and $D(P_{Y|X} \| Q_{Y|X} | P_X)$ are small, $H_{\ell}(Q_{Y|X} | Q_X )$ can closely approximate $\E_P[\ell(Y,\psi_Q(X))]$.
	As mentioned above, other methods developed in Section~\ref{sec:H_diff} can be extended for this purpose as well, and may provide even tighter performance guarantees.
	
	In the special case where $P=P_X P_{Y|X}$ and $Q=P_X Q_{Y|X} $ share the same marginal distribution of $X$, the decision rule $\psi_Q$ defined above preserves its optimality under this new $Q$, and we have the following alternative bounds due to the representation of the conditional entropy via the unconditional entropy in \eqref{eq:cond_uncond} and Lemma~\ref{lm:H_diff_E}.
	\begin{lemma}\label{lm:cond_H_diff_PX}
		Under $P=P_X P_{Y|X}$ and $Q=P_X Q_{Y|X}$, let $P_x \deq P_{Y|X=x}$ and $Q_x \deq Q_{Y|X=x}$.
		Then
		\begin{align}\label{eq:H_cond_diff_PX+}
			H_{\ell}(P_{Y|X} | P_X ) - H_{\ell}(Q_{Y|X} | P_X )
			&\le \int_\sX \big( \E_{P_{x}}[\ell(Y,\psi_Q(x))] - \E_{Q_{x}} [\ell(Y,\psi_Q(x))]  \big) P_X({\rm d}x) 
		\end{align}
		and
		\begin{align}\label{eq:H_cond_diff_PX-}
			H_{\ell}(Q_{Y|X} | P_X ) - H_{\ell}(P_{Y|X} | P_X )
			&\le \int_\sX \big( \E_{Q_{x}} [\ell(Y,\psi_{P}(x))] - \E_{P_{x}}[\ell(Y,\psi_{P}(x))]  \big) P_X({\rm d}x) .
		\end{align}
	\end{lemma}
	\noindent
	With Lemma~\ref{lm:cond_H_diff_PX}, the results developed in Sections~\ref{sec:H_diff_TV} to \ref{sec:H_diff_dlA} on unconditional entropy difference can be directly applied to bounding the conditional entropy difference, by bounding the integrands in \eqref{eq:H_cond_diff_PX+} and \eqref{eq:H_cond_diff_PX-}.
	
	The bounds for conditional entropy difference obtained in Lemma~\ref{lm:cond_H_diff_E} or Lemma~\ref{lm:cond_H_diff_PX} combined with the techniques developed in Section~\ref{sec:H_diff} can provide performance guarantees for decision making with distribution shift: the performance of a decision rule $\psi_Q$ under a new distribution $P$, represented by $\E_P[\ell(Y,\psi_Q(X))]$, may be approximated in terms of its performance under the original distribution $Q$ where it is optimally designed, represented by $H_{\ell}(Q_{Y|X} | Q_X )$.
	As illustrated by the preceding example for $\ell\in[0,1]$, the simple upper bound for the right-hand side of \eqref{eq:cond_H_diff_1} given in \eqref{eq:cond_H_diff_3} is an analogue of the result in \cite[Theorem~1]{Ben-David2010} on binary classification with distribution shift, and is an extension of it to general Bayesian inference problems.

	\subsection{Excess risk bounds via entropy difference}
	Besides comparing $\E_P[\ell(Y,\psi_Q(X))]$ against $H_{\ell}(Q_{Y|X} | Q_X )$, it is also of interest to study the gap between $\E_P[\ell(Y,\psi_Q(X))]$ and $H_{\ell}(P_{Y|X} | P_X )$, which amounts to the excess risk incurred by using $\psi_Q$ under distribution $P$ rather than using the optimal decision rule $\psi_P$.
	The following result, in the same spirit of Lemma~\ref{lm:ERM_Rexcess}, shows that the excess risk can be upper-bounded in terms of the previously developed upper bounds for the conditional entropy difference $| H_\ell(Q_{Y|X}|Q_X) - H_\ell(P_{Y|X}|P_X) |$ or $| H_\ell(Q_{Y|X}|P_X) - H_\ell(P_{Y|X}|P_X) |$.
	\begin{theorem}\label{th:excess_cond}
		The excess risk of using $\psi_Q$, the Bayes decision rule with respect to $(\sA,\ell)$ under $Q=Q_X Q_{Y|X}$, under another distribution $P=P_X P_{Y|X}$ satisfies
		\begin{align}
			\E_P[\ell(Y,\psi_Q(X))] - H_\ell(P_{Y|X}|P_X) \le 2B_Q , \label{eq:excess_cond_Q}
		\end{align}
		where $B_Q$ is any upper bound for $| H_\ell(Q_{Y|X}|Q_X) - H_\ell(P_{Y|X}|P_X) |$ obtained based on Lemma~\ref{lm:cond_H_diff_E}. 
		Additionally, it also holds that
		\begin{align}
			\E_P[\ell(Y,\psi_Q(X))] - H_\ell(P_{Y|X}|P_X) \le 2B_P , \label{eq:excess_cond_P}
		\end{align}
		where $B_P$ is any upper bound for $| H_\ell(Q_{Y|X}|P_X) - H_\ell(P_{Y|X}|P_X) |$ obtained based on either Lemma~\ref{lm:cond_H_diff_E} or Lemma~\ref{lm:cond_H_diff_PX}.
	\end{theorem}
	
	\begin{proof}
		To show \eqref{eq:excess_cond_Q}, we can write the entropy difference $\E_P[\ell(Y,\psi_Q(X))] - H_\ell(P_{Y|X}|P_X) $ as
		\begin{align}
			\big(\E_P[\ell(Y,\psi_Q(X))] - H_\ell(Q_{Y|X}|Q_X) \big) +  \big( H_\ell(Q_{Y|X}|Q_X) - H_\ell(P_{Y|X}|P_X) \big) \label{eq:pf_excess_cond_Q_1} .
		\end{align}
		The claim then follows from the fact that any upper bound for $H_\ell(P_{Y|X}|P_X) - H_\ell(Q_{Y|X}|Q_X)$ obtained based on Lemma~\ref{lm:cond_H_diff_E} also upper-bounds	$\E_P[\ell(Y,\psi _Q(X))] - H_\ell(Q_{Y|X}|Q_X)$.
		
		Next we prove \eqref{eq:excess_cond_P}.
		Adopting the same definitions of $P_x$ and $Q_x$ as in Lemma~\ref{lm:cond_H_diff_PX}, we have
		\begin{align}
			& \E_P[\ell(Y,\psi_Q(X))] - H_\ell(P_{Y|X}|P_X) \nonumber \\
			= & \big(\E_P[\ell(Y,\psi_Q(X))] - H_\ell(Q_{Y|X}|P_X) \big) +  \big( H_\ell(Q_{Y|X}|P_X) - H_\ell(P_{Y|X}|P_X) \big) \label{eq:pf_excess_cond_2} \\
			= &  \int_\sX (\E_{P_x}[\ell(Y,\psi_{Q}(x))] - \E_{Q_x}[\ell(Y,\psi_{Q}(x))] )P_X({\rm d}x) + \big( H_\ell(Q_{Y|X}|P_X) - H_\ell(P_{Y|X}|P_X) \big) \label{eq:pf_excess_cond_3} \\
			\le &  \int_\sX (\E_{P_x}[\ell(Y,\psi_{Q}(x))] - \E_{Q_x}[\ell(Y,\psi_{Q}(x))] )P_X({\rm d}x) + \nonumber \\
			&  \int_\sX (\E_{Q_x}[\ell(Y,\psi_{P}(x))] - \E_{P_x}[\ell(Y,\psi_{P}(x))] )P_X({\rm d}x) \label{eq:pf_excess_cond_4} 
		\end{align}
		where \eqref{eq:pf_excess_cond_3} uses the fact that $\psi_Q$ remains as a Bayes decision rule under the joint distribution $P_X Q_{Y|X}$; and the last step is due to \eqref{eq:H_cond_diff_PX-} in Lemma~\ref{lm:cond_H_diff_PX}.
		
		Note that according to \eqref{eq:cond_H_diff_E+}, any upper bound for $H_\ell(P_{Y|X}|P_X) - H_\ell(Q_{Y|X}|P_X)$ obtained based on Lemma~\ref{lm:cond_H_diff_E} also upper-bounds	$\E_P[\ell(Y,\psi _Q(X))] - H_\ell(Q_{Y|X}|P_X)$.
		It then follows from \eqref{eq:pf_excess_cond_2} that $\E_P[\ell(Y,\psi_Q(X))] - H_\ell(P_{Y|X}|P_X) \le 2 B$ for any upper bound $B$ for $|H_\ell(P_{Y|X}|P_X) - H_\ell(Q_{Y|X}|P_X)|$ obtained by Lemma~\ref{lm:cond_H_diff_E}.
		
		Moreover, any upper bound for $H_\ell(P_{Y|X}|P_X) - H_\ell(Q_{Y|X}|P_X)$ or $H_\ell(Q_{Y|X}|P_X) - H_\ell(P_{Y|X}|P_X)$ obtained based on Lemma~\ref{lm:cond_H_diff_PX} also upper-bounds one of the two integrals in \eqref{eq:pf_excess_cond_4} respectively. It follows that  $\E_P[\ell(Y,\psi_Q(X))] - H_\ell(P_{Y|X}|P_X) \le 2 B$ for any upper bound $B$ for $|H_\ell(P_{Y|X}|P_X) - H_\ell(Q_{Y|X}|P_X)|$ obtained by Lemma~\ref{lm:cond_H_diff_PX}.
		This proves \eqref{eq:excess_cond_P}.
	\end{proof}
	As an example, we can use Theorem~\ref{th:excess_cond} to bound the excess risk in estimating $Y$ from a noisy observation $X$ when the prior distribution of $Y$ is wrongly specified. For instance, when $Y\in\R$ has a prior distribution $P_Y$ and $X={\alpha}Y+V$ with $V\sim\mathcal N(0,1)$ independent of $Y$, if the prior distribution of $Y$ is assumed to be $Q_Y$, then the mismatched Bayes estimator with respect to the quadratic loss is $\psi_Q(x) = {\int y e^{-(x-{\alpha}y)^2/2} Q({\rm d}y)}/{\int  e^{-(x-{\alpha}y')^2/2} Q({\rm d}y')}$ instead of the true Bayes estimator $\psi_P(x)=\E_P[Y|X=x]$ \cite{Verdu_mismatch}.
	The following corollary bounds the excess risk of using a mismatched Bayes estimator in a more general setting.
	\begin{corollary}
		Suppose $Y\in\R$ has a prior distribution $P_Y$, and $X=g(Y,V)$ with some function $g$ and noise $V$ independent of $Y$. Let $\psi_Q$ be the Bayes estimator with respect to the quadratic loss when the prior distribution of $Y$ is assumed to be $Q_Y$ while $X$ is assumed to have the same functional dependence on $Y$ and $V$. Then 
		\begin{align}
			\E_P[(Y-\psi_Q(X))^2] - H_2(P_{Y|X}|P_X) \le 
			& \sqrt{\Var_Q[(Y-\psi_Q(X))^2] \chi^2(P_Y\| Q_Y)} \, +  \nonumber\\
			& \sqrt{\Var_P[(Y-\psi_P(X))^2] \chi^2(Q_Y\| P_Y)} .
		\end{align}
	\end{corollary}
	\begin{proof}
		This result is a slight variation of \eqref{eq:excess_cond_Q}, but we follow the same line of its proof:
		\begin{align}
			& \E_P[(Y-\psi_Q(X))^2] - H_2(P_{Y|X}|P_X) \nonumber \\ 
			=& \E_P[(Y-\psi_Q(X))^2] - \E_Q[(Y-\psi_Q(X))^2] + \E_Q[(Y-\psi_Q(X))^2]- H_2(P_{Y|X}|P_X)  \\
			\le& \big(\E_P[(Y-\psi_Q(X))^2] - \E_Q[(Y-\psi_Q(X))^2]\big) +
			\big(\E_Q[(Y-\psi_P(X))^2]- H_2(P_{Y|X}|P_X) \big) \\
			\le& \sqrt{\Var_Q[(Y-\psi_Q(X))^2] \chi^2(P_{X,Y}\| Q_{X,Y})} + \sqrt{\Var_P[(Y-\psi_P(X))^2] \chi^2(Q_{X,Y}\| P_{X,Y})} \label{eq:pf_co_mismatch1} \\
			=& \sqrt{\Var_Q[(Y-\psi_Q(X))^2] \chi^2(P_Y\| Q_Y)} + 
			\sqrt{\Var_P[(Y-\psi_P(X))^2] \chi^2(Q_Y\| P_Y)} ,
		\end{align}
		where \eqref{eq:pf_co_mismatch1} follows from the same argument as in the proof of Theorem~\ref{prop:cont_Chi2}; and the last step uses the fact that
		$\chi^2(P_{X,Y}\| Q_{X,Y})=\chi^2(P_Y \| Q_Y)$ and $\chi^2(Q_{X,Y}\| P_{X,Y})=\chi^2(Q_Y \| P_Y)$, which follows from the definition of the $\chi^2$ divergence and the fact that $P_{X|Y}$ and $Q_{X|Y}$ are identical and only depend on the distribution of $V$, as a consequence of the assumed form of $X$.
	\end{proof}
	
	Theorem~\ref{th:excess_cond} can also be applied to statistical learning problems where the learned decision rule is optimally designed under a data-dependent distribution $Q$.
	Combined with the results in Section~\ref{sec:H_diff}, it can provide excess risk upper bounds in terms of the statistical distances between $Q$ and the data-generating distribution $P$. We give an example in the next subsection.

	\subsection{Excess risk in learning by projecting to exponential family}
	We now consider a procedure for supervised learning that is different from both the frequentist learning and the Bayesian learning discussed in the previous sections. 
	To precisely describe it, we need the following definitions and properties of exponential family distributions.
	A parametrized family of distributions $\mathcal Q = \{ Q_\theta : \theta \in \mathsf\Theta \subset \R^d\}$ on $\sZ = \sX\times\sY$ is an exponential family if each element can be written as $Q_\theta(z) = \exp\{\theta^\top \varphi(z) - A(\theta)\}$ for some $\theta\in\mathsf\Theta$, with $\varphi:\sZ\rightarrow\R^d$ as a potential function, $A(\theta) \deq \log \int_\sZ \exp\{\theta^\top \varphi(z)\}\nu({\rm d}z)$ as the log partition function, and $\nu$ as a density on $\sZ$.
	For a distribution $P$ on $\sZ$ which may not belong to $\mathcal Q$, its projection to $\mathcal Q$, defined as $\argmin_{Q\in\mathcal Q} D(P \| Q)$, is given by $Q^* \deq Q_{\theta^*}$ with a $\theta^*\in\Theta$ that satisfies 
	\begin{align}\label{eq:exp_theta*}
		\nabla A(\theta^*) = \mu \deq \E_P[\varphi(Z)] .
	\end{align}
	Similarly, given a dataset $Z^n = ((X_1,Y_1),\ldots,(X_n,Y_n))$ drawn i.i.d.\ from $P$, the projection of its empirical distribution $\wh P_n$ to $\mathcal Q$, defined as the solution to the maximum-likelihood estimation $\argmax_{Q\in\mathcal Q} \sum_{i=1}^n \log Q(Z_i)$, is given by $\wh Q \deq Q_{\hat \theta}$ with a $\hat\theta \in \Theta$ that satisfies
	\begin{align}\label{eq:exp_theta_hat}
		\nabla A(\hat\theta) = \hat \mu \deq \frac{1}{n}\sum_{i=1}^{n}\varphi(Z_i) .
	\end{align}
	Define the convex conjugate of $A$ as $A^*(\mu) \deq \sup_{\theta\in\mathsf \Theta} \mu^\top\theta - A(\theta)$ for any $\mu$ that can be written as $\E_{Q_\theta}[\varphi(Z)]$ for some $\theta\in\mathsf\Theta$.
	When $\mathcal Q$ is \emph{minimal}, meaning that $Q_\theta$ and $Q_{\theta'}$ are different for any $\theta\neq\theta'\in\mathsf \Theta$, it is known from convex duality \cite{GM_exp_var} that $\theta^*$ and $\hat\theta$ implicitly defined above can be explicitly written as $\theta^* = \nabla A^*(\mu)$ and $\hat\theta = \nabla A^*(\hat \mu)$.
	Figure~\ref{fig:exp_proj} illustrates the above defined quantities.
	\begin{figure}[h]
		\centering
		\includegraphics[scale = 0.236]{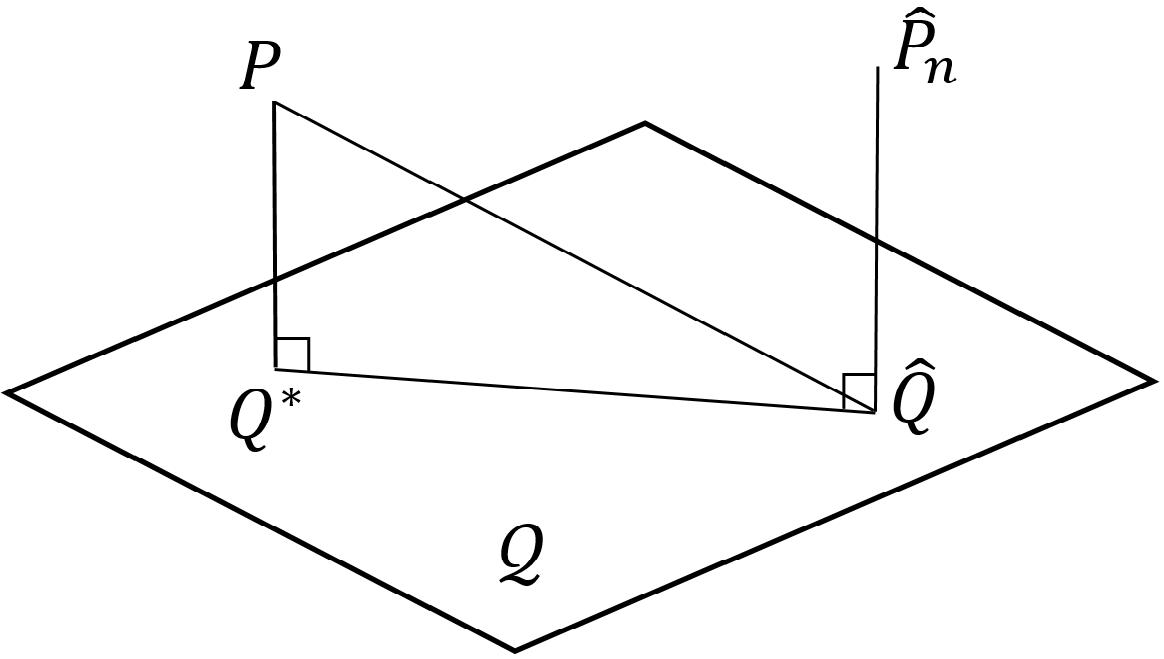}
		\caption{Illustration of the projections of the data-generating distribution $P$ and the empirical distribution $\wh P_n$ to the exponential family $\mathcal Q$. }
		\label{fig:exp_proj}
	\end{figure}

	With the above definitions, the learning procedure under consideration can be described as follows: given a dataset $Z^n$ drawn i.i.d.\ from $P$, first project its empirical distribution $\wh P_n$ to a predefined exponential family $\mathcal Q$ on $\sZ$ to obtain $\wh Q$, then the learned decision rule for predicting $Y$ based on a fresh observation $X$ is taken as the Bayes decision rule $\psi_{\wh Q}$ that is optimal under $\wh Q$.
	The following result based on Theorem~\ref{th:excess_cond} provides upper bounds for its expected excess risk.
	\begin{corollary}\label{co:excess_exp_TV}
		For the learning procedure described above, under the assumptions that $\mathcal Q$ is minimal and that the loss function $\ell$ takes values in $[0,1]$, the expected excess risk of using $\psi_{\wh Q}$ as the learned decision rule under the data-generating distribution $P$ satisfies
		\begin{align}
			\E[\ell(Y,\psi_{\wh Q}(X))] - & H_\ell(P_{Y|X}|P_X) 
			\le 2d_{\rm TV}(P,Q^*) + \nonumber \\
			& \sqrt{2 \|\mu\| \cdot \E \|\nabla A^*(\mu) - \nabla A^*(\hat \mu)\| + 2 \E |A(\nabla A^*(\mu)) - A(\nabla A^*(\hat \mu))| } , \label{eq:excess_exp_TV}
		\end{align}	
		where the expectations are taken over either $(Z^n,Z)$ or $\wh \mu$ with $P$ as the underlying distribution.
	\end{corollary}
	\begin{proof}
		To make use of Theorem~\ref{th:excess_cond}, we first bound the entropy difference.
		For any realization of the dataset $Z^n$,
		\begin{align}
			|H_\ell(\wh Q_{Y|X}|\wh Q_X) - H_\ell(P_{Y|X} | P_X) | & \le d_{\rm TV}(\wh Q, P) 
			\le d_{\rm TV}(P,Q^*) + d_{\rm TV}(\wh Q, Q^*)
		\end{align}
		where the first inequality is due to Lemma~\ref{lm:cond_H_diff_E} and the assumption that $\ell\in[0,1]$ as used in the proof of Theorem~\ref{prop:cont_TV}, while the second inequality is due to the triangle inequality satisfied by the total variation distance.
		Further, 
		\begin{align}
			d_{\rm TV}(\wh Q, Q^*) &\le \sqrt{\frac{1}{2}D(Q^* \| \wh Q)} \label{eq:TV_exp_1} \\
			&= \sqrt{\frac{1}{2} \big( \E_{Q^*}[\varphi(Z)]^\top (\theta^* - \hat \theta) - (A(\theta^*) - A(\hat \theta)) \big) } \label{eq:TV_exp_2} \\
			&\le \sqrt{\frac{1}{2} \big( \|\mu\| \|\theta^* - \hat \theta\| + |A(\theta^*) - A(\hat \theta)| \big) } \label{eq:TV_exp_3} \\
			&= \sqrt{\frac{1}{2} \big( \|\mu\| \|\nabla A^*(\mu) - \nabla A^*(\hat \mu)\| + |A(\nabla A^*(\mu)) - A(\nabla A^*(\hat \mu))| \big) } \label{eq:TV_exp_4}
		\end{align}
		where \eqref{eq:TV_exp_1} uses the Pinsker's inequality; \eqref{eq:TV_exp_2} uses the property of the exponential family distributions; \eqref{eq:TV_exp_3} uses the fact that $\E_{Q^*}\varphi(Z)=\mu$ and the Cauchy-Schwarz inequality; and \eqref{eq:TV_exp_4} uses \eqref{eq:exp_theta*} and \eqref{eq:exp_theta_hat} as well as the assumption that $\mathcal Q$ is minimal so that $(\nabla A)^{-1} \equiv \nabla A^*$.
		It then follows from \eqref{eq:excess_cond_Q} in Theorem~\ref{th:excess_cond} that 
		\begin{align}
			\E[\ell(Y,\psi_{\wh Q}(X))|Z^n] - &H_\ell(P_{Y|X}|P_X) 
			\le 2d_{\rm TV}(P,Q^*) + \nonumber \\
			& \sqrt{2 \big( \|\mu\| \|\nabla A^*(\mu) - \nabla A^*(\hat \mu)\| + |A(\nabla A^*(\mu)) - A(\nabla A^*(\hat \mu))| \big) } 
		\end{align}	
		almost surely for $Z^n$. 
		The claim follows by taking expectations on both sides of the above inequality over $Z^n$ and applying Jensen's inequality on the right-hand side.
	\end{proof}
	Corollary~\ref{co:excess_exp_TV} clearly shows that the excess risk for learning by projecting the empirical distribution to an exponential family consists of two parts: the \emph{approximation error}, represented by the first term on the right-hand side of \eqref{eq:excess_exp_TV}, and the \emph{estimation error}, represented by the second term.
	The approximation error depends on the total variation distance from the data-generating distribution $P$ to the exponential family $\mathcal Q$ and does not depend on the data size. The estimation error on the other hand vanishes as $n$ grows whenever $A$ and $\nabla A^*$ are continuous, which is due to the fact that $\hat\mu\rightarrow\mu$ almost surely as $n\rightarrow\infty$.
	
	The learning procedure considered above can be extended to the cases where the family of distributions $\mathcal Q$ is not predefined, but dependent on the empirical distribution $\wh P_n$, and where the distribution $\wh Q$ under which the learned decision rule is optimally designed is found by other criteria. An example is the recently proposed maximum conditional entropy framework of learning \cite{FarniaTse16}, where $\mathcal Q$ is a set of distributions centered at $\wh P_n$, and $\wh Q$ is chosen to be an element of $\mathcal Q$ with the maximum generalized conditional entropy with respect to some loss function. A special case of this framework with moment-matching conditions to construct $\mathcal Q$ and with the log loss may be interpreted as projecting the empirical conditional distribution $\wh P_{Y|X}$ to an exponential family of conditional distributions associated with a generalized linear model. 
	More generally, the minimax approach to statistical learning where the goal is to find a decision rule that minimizes the worst-case expected loss in $\mathcal Q$, c.f. \cite{FarniaTse16,LR_minimax} and the reference therein, is equivalent to the maximum conditional entropy approach under regularity conditions \cite{FarniaTse16}.
	Whether Theorem~\ref{th:excess_cond}, especially \eqref{eq:excess_cond_P} can be leveraged to analyze the excess risk in the maximum conditional entropy framework of learning would be an interesting research problem.

	\section{Possible improvements and extensions}
	In this work, we have derived upper and lower bounds for the difference of the generalized entropy between two distributions in terms of various statistical distances, and applied the results to the excess risk analysis in three major learning problems.
	In this section we discuss possible improvements and extensions of this work.
	\begin{itemize}[leftmargin=*]
	\item
	Improvement of the entropy difference bound.
	The majority of the entropy difference bounds obtained in Section~\ref{sec:H_diff} are based on Lemma~\ref{lm:H_diff_E}. Only in Section~\ref{sec:H_diff_Breg} we took a different route by considering an exact representation of the entropy difference in terms of a Bregman divergence between the distributions. There is another exact representation of the entropy difference, which can be viewed as a refinement of Lemma~\ref{lm:H_diff_E}:
	\begin{align}\label{eq:H_diff_lbub}
		H_\ell(P) - H_\ell(Q) 
		= \E_P[\ell(Z,a_Q)] - \E_Q[\ell(Z,a_Q)] + \underbrace{\E_P[\ell(Z,a_P)] - \E_P[\ell(Z,a_Q)]}_{-D_{\sA,\ell}(P,Q) \le 0}.
	\end{align}
	The slack of Lemma~\ref{lm:H_diff_E} is clearly seen as the nonnegative $D_{\sA,\ell}(P,Q) \deq \E_P[\ell(Z,a_Q)] - \E_P[\ell(Z,a_P)]$, which can be thought of an $(\sA,\ell)$-specific divergence between $P$ and $Q$ \cite{gunwald2004}.
	A possible way to improve the results obtained based on Lemma~\ref{lm:H_diff_E} is thus to evaluate or lower-bound $D_{\sA,\ell}(P,Q)$.
	
	\item
	Applying Theorem~\ref{th:D_PQ_ell} to learning problems. 
	As shown in Section~\ref{sec:H_diff_P_ell}, Theorem~\ref{th:D_PQ_ell} can potentially provide much tighter entropy difference bounds. The reason is that the loss is real-valued, with a one-dimensional distribution, whereas the data distribution $P$ or $Q$ can be high-dimensional. The difficulty to apply this improvement to frequentist learning problems is that, the empirical loss is a function of non-i.i.d. quantities, as the learned hypothesis depends on the training data. It is thus hard to characterize the resulting distribution of the empirical loss. But this problem can be an interesting future direction of research.
	
	\item
	Continuity of other general definitions of entropy. The generalized entropy considered in this work is a function of \emph{probability distribution} on a sample space. This definition could be further generalized to functions of other quantities of interest, e.g. to the von Neumann entropy (a.k.a.\ quantum entropy) as a function of the density matrix. Such generalization may also be carried out in a decision-making framework \cite{e19050206}. It is therefore of interest to study if the continuity property of other generalized entropies can be useful for analyzing excess risks of the related decision-making or optimization problems.
	\end{itemize}
	
	\appendix
	\setcounter{lemma}{0}
	\renewcommand{\thelemma}{\Alph{section}\arabic{lemma}}
	
	\section*{Appendix}
	\section{Proof of Lemma~\ref{lm:DPQ_variational_gen} }\label{appd:pf_DPQ_variational_gen}
	The Donsker-Varadhan theorem states that
	\begin{align}
		D(P \| Q) = \sup_{g:\sZ\rightarrow\R} \E_P[g(Z)] - \log \E_Q[e^{g(Z)}] .
	\end{align}
	It implies that for any $f:\sZ\rightarrow\R$ and any $\lambda\in\R$,
	\begin{align}
		D(P \| Q) \ge \lambda(\E_P[f(Z)] - \E_Q[f(Z)]) - \log \E_Q[e^{\lambda(f(Z)-\E_Q f(Z))}] .
	\end{align}
	From the assumption that $\log \E_Q[e^{\lambda (f(Z) - \E_Q f(Z) )}] \le \varphi_+(\lambda)$ for all $0\le \lambda < b_+$ and the definition $\varphi_+^*(\gamma) \deq \sup_{0\le \lambda < b_+} \lambda \gamma - \varphi_+(\lambda)$ for $\gamma\in\R$, we have
	\begin{align}
		D(P \| Q) &\ge \sup_{0\le \lambda < b_+} \lambda(\E_P[f(Z)] - \E_Q[f(Z)]) - \varphi_+(\lambda) \\
		&= \varphi_+^*(\E_P[f(Z)] - \E_Q[f(Z)]) .
	\end{align}
	From the definition $\varphi_+^{*-1}(x) \deq \sup\{\gamma\in\R:\varphi_+^*(\gamma)\le x\}$ for $x\in\R$, we have
	\begin{align}
		\E_P[f(Z)] - \E_Q[f(Z)] \le \varphi_+^{*-1}(D(P \| Q)) ,
	\end{align}
	which proves \eqref{eq:DPQ_variational_gen_+}.
	
	The Donsker-Varadhan theorem also implies that for any $f:\sZ\rightarrow\R$ and any $\lambda\in\R$,
	\begin{align}
		D(P \| Q) \ge \lambda(\E_Q[f(Z)] - \E_P[f(Z)]) - \log \E_Q[e^{-\lambda(f(Z)-\E_Q f(Z))}] .
	\end{align}
	From the assumption that $\log \E_Q[e^{-\lambda (f(Z) - \E_Q f(Z) )}] \le \varphi_-(\lambda)$ for all $0\le \lambda < b_-$ and the definition $\varphi_-^*(\gamma) \deq \sup_{0\le \lambda < b_-} \lambda \gamma - \varphi_-(\lambda)$ for $\gamma\in\R$, we have
	\begin{align}
		D(P \| Q) &\ge \sup_{0\le \lambda < b_-} \lambda(\E_Q[f(Z)] - \E_P[f(Z)]) - \varphi_-(\lambda) \\
		&= \varphi_-^*(\E_Q[f(Z)] - \E_P[f(Z)]) .
	\end{align}
	From the definition $\varphi_-^{*-1}(x) \deq \sup\{\gamma\in\R:\varphi_-^*(\gamma)\le x\}$ for $x\in\R$, we have
	\begin{align}
		\E_Q[f(Z)] - \E_P[f(Z)] \le \varphi_-^{*-1}(D(P \| Q)) ,
	\end{align}
	which proves \eqref{eq:DPQ_variational_gen_-}.
	
	The assumption that $\varphi_+(\lambda)$ is strictly convex over $[0, b_+]$ and $\varphi_+(0) =  \varphi_+'(0)=0$ implies that its Legendre dual $\varphi_+^*(\gamma)$ is strictly increasing over $\gamma\ge 0$ and $\varphi_+^*(0)=0$. 
	In addition, the fact that $\varphi_+^*(\gamma)$ is convex over $\gamma\ge 0$ implies that it is continuous over $\gamma\ge 0$.
	Together these imply that $\varphi_+^{*-1}(x)$ is strictly increasing and continuous over $x\ge 0$, and $\varphi_+^{*-1}(0)=0$.
	It follows that $\lim_{x\downarrow 0}\varphi_+^{*-1}(x)=0$.
	The same argument can be used to show that if $\varphi_-(\lambda)$ is strictly convex over $[0, b_-]$ and $\varphi_-(0) =  \varphi_-'(0)=0$, then $\lim_{x\downarrow 0}\varphi_-^{*-1}(x)=0$.

	\section*{Acknowledgment}
	The author would like to thank Prof. Maxim Raginsky, Prof. Yihong Wu, and Jaeho Lee for helpful discussions.
	The author is also thankful to the area editor and anonymous reviewers of the IEEE Transactions on Information Theory, their comments greatly improved the quality of the paper.

	\bibliographystyle{IEEEtran}
	\bibliography{aolin}
	

\end{document}